\documentclass[nohyperref]{article}

\PassOptionsToPackage{sort,compress}{natbib}
\usepackage{natbib}

\usepackage{microtype}
\usepackage{graphicx}
\usepackage{subfigure}
\usepackage{booktabs}
\usepackage{hyperref,etoc}
\usepackage[accepted]{icml2023}
\usepackage{amsmath}
\usepackage{amssymb,mathrsfs}
\usepackage{bm,bbm}
\usepackage{mathtools}
\usepackage{amsthm,threeparttable}
\usepackage[capitalize,noabbrev]{cleveref}
 \usepackage{multicol}
 \usepackage{multirow}
\usepackage{pifont}  % %  11/2020: this is for including a cross, i.e. match the checkbox.
\newcommand{\cmark}{\ding{51}}%
\newcommand{\xmark}{\ding{55}}%
\usepackage{amsbsy}  % % 11/2020: this is for bold checkmark
\usepackage[normalem]{ulem} % % 13/10/21, grigoris: For striking through text, sout command.
 \usepackage{algorithm}
\usepackage{enumitem} 
\usepackage{schemata}
\newcommand\AB[2]{\schema{\schemabox{#1}}{\schemabox{#2}}}

\DeclareMathOperator*{\argmax}{argmax}
\DeclareMathOperator*{\argmin}{argmin}

\definecolor{fhcolor}{rgb}{0.523, 0.235, 0.625}

\theoremstyle{plain}
\newtheorem{theorem}{Theorem}
\newtheorem{case}{Case}
\newtheorem{proposition}{Proposition}
\newtheorem{lemma}{Lemma}

\theoremstyle{definition}
\newtheorem{definition}{Definition}
\newtheorem{assumption}{Assumption}
\theoremstyle{remark}

\crefname{section}{sec.}{sec.} % % 24/1/2022, grigoris: https://tex.stackexchange.com/a/144765

\usepackage[textsize=tiny]{todonotes}

\icmltitlerunning{What can online reinforcement learning benefit from general coverage conditions?}

\begin{document}

\etocdepthtag.toc{mtchapter}
\etocsettagdepth{mtchapter}{subsection}
\etocsettagdepth{mtappendix}{none}

\twocolumn[
\icmltitle{What can online reinforcement learning with function approximation benefit from general coverage conditions?}
%\icmlsetsymbol{equal}{*}

\begin{icmlauthorlist}
\icmlauthor{Fanghui Liu}{yyy}
\icmlauthor{Luca Viano}{yyy}
\icmlauthor{Volkan Cevher}{yyy}
\end{icmlauthorlist}

\icmlaffiliation{yyy}{Laboratory for Information and Inference Systems, 
	 \'{E}cole Polytechnique F\'{e}d\'{e}rale de Lausanne (EPFL), Switzerland}

\icmlcorrespondingauthor{Fanghui Liu}{fanghui.liu@epfl.ch}
%\icmlcorrespondingauthor{Firstname2 Lastname2}{first2.last2@www.uk}

\icmlkeywords{Machine Learning, ICML}

\vskip 0.3in
]

\printAffiliationsAndNotice{}

\begin{abstract}
%Nevertheless, sample-efficient algorithms for online RL has been well studied in general function approximation without the data coverage concept, so it is unclear to us which benefits of data coverage can be brought into online RL.

%General function approximation in online reinforcement learning (RL) often requires certain structural assumptions on Markov decision processes (MDPs). 
%Recently, \citet{xie2023role} pointed out that the structural assumption can be substituted by the coverage condition, originally used in offline RL community, but can still ensure sample-efficient guarantees in online RL.
In online reinforcement learning (RL), instead of employing standard structural assumptions on Markov decision processes (MDPs), using a certain coverage condition (original from offline RL) is enough to ensure sample-efficient guarantees \cite{xie2023role}.
In this work, we focus on this new direction by digging more possible and general coverage conditions, and study the potential and the utility of them in efficient online RL.
We identify more concepts, including the $L^p$ variant of concentrability, the density ratio realizability, and trade-off on the partial/rest coverage condition, that can be also beneficial to sample-efficient online RL, achieving improved regret bound.
Furthermore, if exploratory offline data are used, under our coverage conditions, both \emph{statistically} and \emph{computationally} efficient guarantees can be achieved for online RL.
Besides, even though the MDP structure is given, e.g., linear MDP, we elucidate that, good coverage conditions are still beneficial to obtain faster regret bound beyond $\widetilde{\mathcal{O}}(\sqrt{T})$ and even a \emph{logarithmic} order regret.
These results provide a good justification for the usage of general coverage conditions in efficient online RL.

\end{abstract}

\section{Introduction}

Modern reinforcement learning (RL) algorithms modeled by Markov Decision Processes (MDPs) \cite{szepesvari2010algorithms}, e.g., deep Q network \cite{mnih2015human}, Go \cite{silver2016mastering}, often work in an online setting under large (or even infinite) state space and action space.
Here the terminology \emph{online} means that the agent repeatedly interacts with the environment by executing a policy and observing the past trajectory.
To tackle the large state/action space setting, function approximation \cite{sutton1999policy,jin2020provably} is a powerful and indispensable technique in both theory and practice to approximate the true value function from a pre-given function class.

%Our result is based on all-policy concentrability which requires a highly exploratory dataset.
%As suggested by \citep[Theorem 4]{chen2019information}, it also implicitly imposes structural assumptions on the MDP dynamics.
%Accordingly, the following proposition demonstrates that we can get rid of the MDP structure assumption in \cite{song2022hybrid} if we use the same concentrability coefficient.

In online RL, much efforts are devoted to developing \emph{sample-efficient} algorithms in a \emph{general} function approximation class beyond linear MDP \cite{jin2020provably}.
By explicitly assuming some structural assumptions on MDPs, e.g., Eluder dimension \cite{russo2013eluder}, Bellman Eluder (BE) dimension \cite{jin2021bellman}, typical algorithms including GOLF \cite{jin2021bellman}, OPERA \cite{chen2022general} enjoy sample efficient guarantees in online RL.
Normally, these algorithms for general function approximation in online RL are not \emph{computation-efficient} due to the constructed “global” confidence sets for exploration.

Instead of \emph{explicitly} assuming structural assumptions as above-mentioned, the data coverage condition \cite{munos2008finite}, widely used in offline RL, in fact \emph{implicitly} imposes structural assumptions on the MDP dynamics \cite{chen2019information}.
It asserts that a pre-given (even unknown) data distribution $\mu$ provides sufficient coverage over the state space. 
Recently, \citet{xie2023role} show that, even though no offline data are used, a good data coverage condition, w.r.t an underlying distribution, can ensure sample-efficient guarantees in online RL.
%That means, coverability can be regarded as an intrinsic structural properties of MDP by restricting the complexity of all possible state transitions in some sense.
%Based on this, if offline data are further used, such coverage condition as well as offline data (and MDP structural assumptions) can make online RL efficient both statistically and computationally \cite{song2022hybrid}.

Accordingly, studying the coverage condition instead of classical structural assumption on MDPs in online RL is an alternative but promising way.
This direction provides a natural connection between offline and online RL in both theory \cite{wang2021statistical,foster2021statistical,zanette2021exponential} and practice \cite{nair2020accelerating,levine2020offline}.
Besides, it provides a new view to develop \emph{sample-efficient} and even \emph{computation-efficient} algorithms for general function approximation in online RL, as suggested by \cite{song2022hybrid}.

%In fact, the relationship between offline and online RL has been widely studied, e.g., separation in sample complexity \cite{wang2021statistical}, 
%offline data making online RL efficient \cite{song2022hybrid} both statistically and computationally.

In this work, we focus on this new direction by digging more general coverage conditions, and study the potential and the utility of them in efficient online RL under various scenarios as below.
%Since there are numerous types of coverage conditions, it is natural to wonder which coverage condition can ensure sample-efficient online RL.
%We further decouple the effect if partial coverage condition is employed.

As a starting point, in Section~\ref{sec:lpcw}, we identify more concepts of coverage conditions to ensure sample-efficient online RL, including the $L^p$ variant of concentrability and the density ratio realizability. We take the $L^{p}(\mathrm{d}\mu)$ space with $p \geqslant 1$ measure as an example to combine them.
Our general coverage conditions can ensure the typical GOLF algorithm \cite{jin2021bellman} to achieve sample-efficient guarantees in online RL. 
To further obtain computation-efficient efficiency, the offline data can be used for exploration, see a typical hybrid-Q algorithm \cite{song2022hybrid}.
Under this setting, our coverage condition is still useful to ensure both \emph{statistically} and \emph{computationally} efficient guarantees for online RL with general function approximation.
By doing so, the required structure assumption in \cite{song2022hybrid}, e.g., Bellman rank, BE dimension, can be substituted by our coverage condition.
This utility of coverage conditions supports our target in this work. 
%Our analysis does not require any structure assumption on MDP, which is needed in \cite{song2022hybrid}.

In Section~\ref{sec:disen}, based on our coverage condition, we decouple the all-policy coverage condition into a partial-policy coverage condition by some (unknown) data distribution and the rest-policy coverage condition, which is quite realistic in practice. We theoretically prove that the trade-off on the partial/rest coverage condition, are able to obtain a better regret bound than \cite{xie2023role}.
This provides a good justification on the study of general coverage conditions.

In Section~\ref{sec:linearmdp}, we also identify that, even if the MDP structure is given, e.g., linear MDP, our coverage conditions are still useful. We demonstrate that, the typical LSVI-UCB algorithm in linear MDP \cite{jin2020provably} equipped with certain coverage conditions is able to obtain faster regret bound than $\mathcal{O}(\sqrt{T})$, and even $\mathcal{O}(\log T)$ regret.

{\bf Technical contributions:} 
In this paper, we give an affirmative answer to identify more general coverage concepts for improved efficient online RL under several scenarios.
We follow the proof framework of \cite{xie2023role} on the regret analysis and the decomposition of the on-policy average Bellman error. 
The technical contributions of this work mainly lie in 1) under this framework, how to tackle the unbounded on-policy measure under the setting of partial/rest coverage trade-off in Section~\ref{sec:disen}; 2) providing a new proof framework for linear MDP by building the connection between the on-policy measure and the underlying distribution for improved regret bounds in Section~\ref{sec:linearmdp}.

%Note that, our goal is not to design a new algorithm for analysis but identify the benefits of coverage conditions in online RL when compared to the standard structural assumptions, which is the first but significant step via an in-depth theoretical understanding of various coverage concepts.
{\bf Goal of this paper:} This paper does not contribute to design a new algorithm but provides a possibility to substitute structural assumptions by coverage conditions.
We identify more general coverage conditions and dig several good examples for improved efficient online RL.
Our analysis sheds light on the utility of coverage conditions in online RL, which could open the door to design new efficient algorithms from offline to online RL in practice motivated by our theoretical results.
It bridges the study of offline and online RL and is important for the study of hybrid RL.

In fact, coverage condition is an intrinsic structural property of MDPs that describes the complexity of probability transitions, which does not involve additional information when compared to structural assumptions on MDP.
Nevertheless, we do not claim that coverage conditions are better than certain structural assumptions on MDPs in online RL.
The relationship between them requires a refined analysis but is beyond the scope of this work.
%In fact, it can be expected that RL in the offline and online setting can be naturally connected via coverage conditions.

%We expect that our analysis opens the door to understand the benefit of data coverage and looking forward (and perhaps more importantly) new practical algorithms designed by our results. 

%We also need to remark that, in this work, we do not provide a new algorithm for analysis. This is not the goal of this work.
%Instead, the main target of our work is to identify the benefits of data coverage in online RL.
%By analysing the current online RL algorithms, xxx

%\vspace{-0.2cm}
\section{Preliminaries and related work}
\label{sec:related_work}
%\vspace{-0.15cm}

We start with introducing basic concepts of online and offline RL \cite{sutton2018reinforcement}, and then give an overview of function approximation in RL.

{\bf Notations:} We use $[T]$ as a shorthand of $\left \{ 1,2,\dots ,T \right \}$ for any positive integer $T$. Define the Lebesgue space $L^p(\mathbb{R}^d)$ with its norm $\| f \|_{L^p} = \int_{\mathbb{R}^d} f(\bm x) \mathrm{d} \bm x$, and the $L^p(\mathrm{d}\mu)$ space with its norm $\| f \|^p_{L^p(\mathrm{d} \mu)} = \int [f(\bm x)]^p \mathrm{d} \mu $ over the probability measure $\mu$. Here we assume $p \geqslant 1$.
A typical example is the $L^2(\mathrm{d}\mu)$ space, a Hilbert space, that is commonly used in learning theory.
The notation $\widetilde{\mathcal{O}}$ omits the logarithmic factor.

\subsection{Basic concepts}
{\bf Markov decision processes} (MDPs): In our work, we consider a finite-horizon episodic MDP, denoted as $\text{MDP}(\mathcal{S}, \mathcal{A}, H, \mathbb{P}, r)$ setting to model reinforcement learning, where $\mathcal{S}$ is the state space with potentially infinite states; $\mathcal{A}$ is the finite action space; $H$ is the number of steps in one episode; $\mathbb{P} := \{ \mathbb{P}_h \}_{h=1}^H$ is defined as the transition probability $\mathbb{P}_h ( s_{h+1} | s_h, a_h) $ from the current state-action pair $(s_h,a_h)$ to the next state $s_{h+1} \in \mathcal{S}$ for every $h \in [H]$; We use $r := \{ r_h \}_{h=1}^H $ to denote the reward $r_{h}(s,a)$ received at each $h \in [H]$ when taking the action $a$ at state $s$.
For ease of description, we assume the reward is non-negative and $\sum_{h=1}^H r_h(s_h,a_h) \in [0,1]$ for any possible trajectory.

A non-stationary policy $\pi$ is a sequence of functions $\pi:= \{ \pi_h: \mathcal{S} \rightarrow {\mathcal{A}} \}_{h=1}^H$, where $\pi_h$ specifies a strategy at step $h$, and induces a distribution over trajectories $\{ (s_h,a_h,r_h) \}_{h=1}^H$ by the following process: taking an action $a_h \sim \pi(\cdot| s_h)$, observing a reward $r_h(s_h,a_h)$, and obtaining $s_{h+1} \sim \mathbb{P}_h(\cdot \mid s_h, a_h)$.
We denote $\mathbb{E}_{\pi}[\cdot]$ as the expectation w.r.t the 
randomness of the trajectory $\{(s_h,a_h)\}_{h=1}^H$ generated by
the policy $\pi $, and $\mathrm{Pr}^{\pi}[\cdot]$ as the probability under this process.
Accordingly, the occupancy measure for a policy $\pi$ is defined as 
\begin{equation*}
    \rho_h^{\pi} (s,a) := \mathrm{Pr}^{\pi}[s_h=s, a_h=a], \quad \rho_h^{\pi} (s) := \mathrm{Pr}^{\pi}[s_h=s]\,.
\end{equation*}

The performance of the agent is captured by the \emph{value function}. To be specific, given a policy $\pi$, the (state) value function $V_h^{\pi}\colon \mathcal{S} \to [0,1]$ is defined as the expected cumulative rewards of the MDP starting from step $h \in [H]$
\begin{equation*}
    V_h^\pi(s) :=  \mathbb{E}_{\pi} \left[\sum_{h' = h}^H r_{h'}(s_{h'},  a_{h'} )  \big|  s_h = s \right] \,.
\end{equation*}
Similarly, the action-value function  
$Q_h^\pi:\mathcal{S} \times \mathcal{A} \to [0,1]$ for a policy $\pi$ is defined as 
\begin{equation*}
    Q^{\pi}_h(s,a) :=\mathbb{E}_{\pi} \left[ \sum_{h'=h}^H r_{h'} (s_{h'},a_{h'} ) \,\big|\, s_h=s,\,a_h=a  \right]\,.
\end{equation*}
Since the episode length and the size of action space are both finite, there always exists an optimal policy $\pi^\star = \{ \pi^{\star}_h \}_{h=1}^H$ \cite{puterman2014markov} such that $V^{\pi^\star}_{h}(s) = \sup_{\pi} V_h^\pi(s)$ for all $s\in \mathcal{S}$ and $h\in [H]$. 
For notational simplicity, we abbreviate $V^{\pi^\star}_{h}$ as $V^{\star}_{h}$ and $Q^{\pi^\star}_{h}$ as $Q^{\star}_{h}$.
For a sequence of value functions $\{ Q_h \}_{h=1}^H$, the Bellman operator at step $h$ for a function $f: \mathcal{S} \times \mathcal{A} \rightarrow \mathbb{R}$ is 
\begin{equation*}
    ( \mathcal{T}_h f ) (s,a) = r_h(s,a) + \mathbb{E}_{s' \sim \mathbb{P}_h(\cdot | s, a)} [\max_{a' \in \mathcal{A}} f(s',a')]\,.
\end{equation*}
 We denote $f_h - \mathcal{T}_h f_{h+1}$ as the Bellman error (or Bellman residual).

The target of an RL algorithm is to find an $\epsilon$-optimal policy such that $V^{\star}_{1} (s_1) - V^{\pi}_{1} (s_1)  \leqslant \epsilon$.
In {\bf online RL}, suppose that an agent interacts with the environment for $T$ episodes, the goal is to learn the optimal policy $\pi^\star$ by minimizing the cumulative regret
\begin{equation*}
    {\tt Regret}(T) := \sum_{t=1}^T [ V^{\star}_{1} (s_1) - {V}^{{\pi}^t}_{1} (s_1) ]\,.
\end{equation*}
In {\bf offline RL}, the agent cannot interact with the environment. Instead, at each step $h$, what we have is an offline dataset with $n_{\mathrm{off}}$ samples $\{ (s_h,a_h,r_h,s_{h+1}) \}$: sampling $(s_h,a_h) \overset{iid}{\sim} \mu_h$, receiving the reward $r_h(s,a)$, and  
$s_{h+1} \sim \mathbb{P}_h(\cdot \mid s_h, a_h)$, where offline data distributions are defined by a collection of data distribution $\mu:= \{ \mu_h \}_{h=1}^H$.
The goal of offline RL is to use this offline dataset to learn an $\epsilon$-optimal policy.

{\bf Function approximation:} The target of function approximation in RL is to get rid of the size of the state space.
A typical setting is under the value-based function approximation, where we approximate the value functions for the underlying MDP by a pre-given function class $\mathcal{F} = \mathcal{F}_1 \times \dots \times \mathcal{F}_H$ with $\mathcal{F}_h \subset \{ f \in \mathcal{S} \times \mathcal{A} \rightarrow [0,1]$\}.
One can see that, a basic assumption in function approximation is to describe the size of $\mathcal{F}$, which aims to assert that $\mathcal{F}$ is large enough or complete to cover value functions under transition dynamics.
For notational simplicity, we define $f:= \{ f_h \}_{h=1}^H$ and accordingly $\pi^f$ to be the greedy policy w.r.t., $f$, which takes the action as $\pi_h^f(s) = \argmax_{a \in \mathcal{A}} f_h(s,a)$.
Since no reward is collected at the $(H+1)$-th step, we always set $f_{H+1}=0$.

For each episode $t$, we define the Bellman error  $\delta^{(t)}_h(\cdot,\cdot) := f_h^{(t)}(\cdot,\cdot) - (\mathcal{T}_h f_{h+1}^{(t)})(\cdot,\cdot)$ at step $h$ induced by $f^{(t)} \in \mathcal{F}$.

\subsection{Related works on function approximation}
Here we give an overview of recent works in function approximation under online RL, offline RL, and a hybrid setting, respectively.

{\bf Online RL:} The agent under the online setting requires \emph{exploration} schemes when interacting with the unknown environment.
The simplest scheme is $\epsilon$-greedy, i.e., randomly selecting new actions with $\epsilon$ probability. 
Though computational efficient, this scheme is demonstrated to be statistically inefficient in theory \cite{jin2015low,dann2022guarantees,liu2022understanding}.
Most literature work with ``optimism in the face of uncertainty'' principle for efficient exploration schemes, e.g., upper confidence bound (UCB)-type algorithms \cite{jin2020provably} and Thompson sampling \cite{russo2018tutorial,agrawal2012analysis}.
They have been applied to linear MDP \cite{jin2020provably,yang2020reinforcement}, kernel MDP \cite{yang2020function}, linear mixture MDP \cite{ayoub2020model,zhou2021nearly}.
For general function approximation, under proper assumptions on MDP structure, e.g., Bellman rank \cite{jiang2017contextual}, Eluder dimension \cite{russo2013eluder}, Bilinear rank \cite{du2021bilinear}, BE dimension \cite{jin2021bellman}, admissible Bellman characterization class \cite{chen2022general}, decision-estimation coefficient class \cite{foster2021statistical}, and sequential exploration coefficient \cite{xie2023role}, sample-efficient algorithms based on optimistic principles are designed to ensure statistical efficiency but computational efficiency guarantees are often unattainable.

{\bf Offline RL:} 
The agent under the offline RL setting \cite{levine2020offline} does not interact with the environment and just learns policies solely from a given offline dataset.
This is the intrinsic difference from online RL.
In this case, there is no possibility to do \emph{exploration} but a \emph{data coverage} condition over the offline dataset is required for statistical guarantees.
It requires the dataset to contain any possible state, action pair or trajectory with a lower bounded probability.
A typical example is all-policy concentrability \cite{munos2008finite,zhang2020variational}, which requires the sufficient coverage of offline data over all (relevant) states and actions.
Recent works focus on relaxation from such strong condition from full coverage to partial coverage \cite{uehara2022pessimistic}, and even single-policy concentrability \cite{rashidinejad2021bridging,zhan2022offline} by preventing the policy from visit states and actions where the offline data coverage is poor \cite{liu2020provably} or relying on the principle of ``pessimism'' \cite{xie2021bellman,jin2021pessimism}.

{\bf Online RL with offline data:}
Empirical results work in this setting and have demonstrated the success of offline data \cite{rajeswaran2017learning}, but under certain settings, offline data does not yield statistical improvements in tabular MDPs \cite{xie2021policy}.
Recent work focus on digging the benefit of offline data in online RL, including computation efficiency \cite{song2022hybrid} and sample efficiency \cite{wagenmaker2022leveraging}.

Besides, \citet{xie2023role} demonstrate that, the data coverage condition is able to ensure sample-efficiency in online RL though no offline data is required to be accessed.
This provides a bridge between the analysis techniques of offline and online RL.

\subsection{Coverage conditions}
Here we briefly introduce mathematical concepts of data coverage conditions.

A concept crucial to our discussions is the marginalized importance weights,
which aims to measure the distribution shift from an arbitrary distribution (here we use the occupancy measure by any policy $\pi$) $\rho^{\pi}:= \{ \rho^{\pi}_h \}_{h=1}^H$ to the data distribution $\mu := \{\mu_h\}_{h=1}^H$.
Define $w_{h, \pi/\mu}(s,a) := \frac{\rho_h^{\pi}(s,a)}{\mu_h(s,a)}$ if $\mu_h(s,a) \neq 0$, and then the commonly used concentrability coefficient for all policy in a policy class $\Pi$ \cite{munos2008finite,chen2019information} 
\begin{equation*}
    C_{\infty} := \max_{\pi \in \Pi, h \in [H]} \| w_{h, \pi/\mu} \|_{\infty} \leqslant \| w_{h, \pi/\mu} \|^2_{L^2({\mathrm{d}\mu})}\,,
\end{equation*}
where the $L^2({\mathrm{d}\mu})$ version is developed in \cite{xie2020q}.
For single-policy concentrability, only $\pi^{\star}$ instead of all possible $\pi \in \Pi$ is taken part in these concentrability coefficients \cite{uehara2022pessimistic}.

Concentrability coefficients can be also conducted by Bellman error, e.g., \cite{xie2021bellman}. Here we give an example from \cite{song2022hybrid} by denoting $\delta_h := f_h - \mathcal{T}_h f_{h+1}$ such that
 \begin{equation}\label{eq:Cpi}
        C_{\pi} := \max_{f \in \mathcal{F}} \frac{|[\mathbb{E}_{\rho_h^{\pi}} \delta_h(s,a)]|}{\sqrt{\mathbb{E}_{\mu_h} [\delta_h(s,a)]^2}}\,, \forall \pi \in \Pi\,,
    \end{equation}
which can be upper bounded by the coverability coefficient $C_{\infty}$.
Recently another coverability coefficient is defined as below to ensure sample-efficient exploration in online RL.  
\begin{definition}\citep[Coverability for online RL]{xie2023role}
\label{def:ccov}
The coverability coefficient $C_{\tt cov}$ is for a policy class $\Pi$
\begin{equation*}
\begin{split}
    C_{\tt cov} &:= \inf_{\mu_1,\ldots,\mu_H\in\Delta(\mathcal{S} \times \mathcal{A})}  \sup_{\pi \in \Pi,h\in [H] }\,
  \left\|\frac{\rho_h^\pi}{\mu_h}\right\|_{\infty} \\
  & = \max_{h\in [H] }  \sum_{(s,a) \in \mathcal{S} \times \mathcal{A}}\sup_{\pi \in \Pi} \rho_h^\pi(s,a) \,.
\end{split}
\end{equation*}
It is demonstrated to be equivalent to the cumulative reachability (see the second equality), refer to \citep[Lemma 3]{xie2023role} for details.
\end{definition}

Besides, in online RL, the ``uniformly excited feature" assumption \cite{abbasi2019politex,lazic2020maximum,hao2021online} is commonly used in reinforcement learning theory.
It requires that every occupancy measure induced by a policy $\pi$ yields a positive definite feature covariance matrix such that 
$ \mathbb{E}_{\rho_h^\pi} [\phi(s,a) \phi(s,a)^{\!\top}] \succcurlyeq c \bm I$ for some constant $c > 0$ where $\phi(s,a)$ is the corresponding feature mapping.
By doing so, each policy $\{\pi_h\}_{h=1}^H$ explores uniformly well in the feature space.
This assumption is also used in offline RL but the expectation is taken as a data distribution $\mu$, see feature coverage condition in 
 \citep[Assumption 2]{wang2021statistical}.

\subsection{Basic assumptions}
\label{sec:assum}

Our work focuses on general function approximation in online RL, which is based on the following two standard and commonly-used assumptions in reinforcement learning theory \cite{wang2020reinforcement,jin2021bellman,chen2022general,xie2023role}.
\begin{assumption}[Realizability]\label{assum:rea}
For a hypothesis class $\mathcal{F}$, we assume $Q^{\star}_h \in \mathcal{F}_h$ for any $h \in [H]$.  
\end{assumption}
Define $\mathcal{T}_h \mathcal{F}_{h+1}$ as $\{ \mathcal{T}_h f_{h+1} : f_{h+1} \in \mathcal{F}_{h+1}  \}$, we require the function class $\mathcal{F}$ to be closed under the Bellman operator $\mathcal{T}_h$ as below.
\begin{assumption}[Bellman completeness]\label{assum:bellman}
    For a hypothesis class $\mathcal{F}$, we assume $\mathcal{T}_h \mathcal{F}_{h+1} \in \mathcal{F}_h$ for any $h \in [H]$.  
\end{assumption}

If the function class $\mathcal{F}$ has finite elements, we can directly use its cardinality to measure its ``size".
If $\mathcal{F}$ has infinite elements, the covering number is needed to describe the ``size" of $\mathcal{F}$.
\begin{definition}
    [Covering number \cite{van1996weak}] 
    The $\epsilon$-covering number $\mathscr{N}(\epsilon, \mathcal{F}, \| \cdot \|_{\infty})$ for a function class $\mathcal{F}$ with respect to the metric $\| \cdot \|_{\infty}$ is the minimal number of balls with radius $\epsilon$ measured by $\| \cdot \|_{\infty}$-norm needed to cover the space $\mathcal{F}$. 
For short, we denote $\mathscr{N}(\epsilon, \mathcal{F}, \| \cdot \|_{\infty})$ as $\mathscr{N}_{\mathcal{F}}(\epsilon)$ by omitting  $\| \cdot \|_{\infty}$.
\end{definition}
\vspace{-0.2cm}
%\section{A warm-up of $L^p$ coverage conditions}
%\label{sec:main_result}

%We adopt the shorthand $\rho_h^{(t)} := \rho_h^{\pi^{(t)}}$, and we define
%\begin{equation*}
%\label{eq:def_dbar}
%  \tilde{\rho}_h^{(t)} (s,a) := \sum_{i = 1}^{t - 1} \rho_h^{(i)} (s,a)\,.
%\end{equation*}
%which is the unnormalized average of all state visitations encountered prior to step $t$.

%In this section, we start with the basic assumptions that are commonly used for function approximation for online RL in Section~\ref{sec:assum}, and then give the definition of the coverage coefficient in the $L^p$ ($p \geqslant 1$) space with sample-efficient guarantees in Section~\ref{sec:lpcw}. This is a good starting point and will be helpful to define the related coverage condition over the partial/rest policies in Section~\ref{sec:disen}, that is our main interest for sample-efficient online RL.

\section{Warm-up: Coverage conditions in $L^p$ spaces}
\label{sec:lpcw}

We give the definition of coverability coefficient in the $L^p$ space. %which covers the $L^p$ variant of concentrability and the density ratio realizability. 
In Section~\ref{sec:cwrl}, we demonstrate that, these coverage conditions are able to obtain better regret bound for sample-efficient online RL with general function approximation when compared to \cite{xie2023role}.
Furthermore, under our coverage conditions, computational efficiency can be even achieved if exploratory offline data are used in Section~\ref{sec:offline}.

%More importantly, this is a good starting point and will be helpful to define the related coverage condition over the partial/rest policies in our later analysis.

\subsection{Improved sample-efficient online RL}
\label{sec:cwrl}

\begin{definition}[$L^p$ coverability coefficient]
\label{def:ccw}
Given a policy class $\Pi$, there exists a underlying distribution
$\mu=\{ \mu_h \}_{h=1}^{H}$ admitting $\sum_{(s,a)} \sqrt{\mu_h(s,a)} < \infty$, for any $p \geqslant 1$, the coverability coefficient $C_{\tt cw}$ defined in the $L^p$ space is
\begin{align*}
  C_{\tt cw} := \inf_{\mu_1,\ldots,\mu_H\in\Delta(\mathcal{S} \times \mathcal{A})}  \sup_{\pi \in \Pi,h\in [H] }\,
  \left\|\frac{\rho_h^\pi}{\mu_h}\right\|^p_{L^p(\mathrm{d}\mu_h)}\,.
\end{align*}
\end{definition}
{\bf Remark:} This definition simply extends the application scope of $C_{\tt cov}$ from the $L^{\infty}$ space to the $L^p$ space. One interesting thing is, we only require $\sum_{(s,a)} \sqrt{\mu_h(s,a)} < \infty$ rather than $\sum_{(s,a)} [\mu_h(s,a)]^{1/p} < \infty$, which makes the underlying distribution $\mu$ more general.

It is clear that $C_{\tt cw} \leqslant |\mathcal{S}| |\mathcal{A}|$ if we take $\mu$ is a uniform measure.
The relationship between $C_{\tt cw}$ and $C_{\tt cov}$ can be built by the following lemma, deferred the proof to Appendix~\ref{app:cwcov}.
\begin{lemma}\label{prop:cwcov}
Based on the definition of $C_{\tt cw}$ and $C_{\tt cov}$ in Definition~\ref{def:ccw} and Definition~\ref{def:ccov}, respectively, we have
    \begin{equation*}
      C^{\frac{1}{p}}_{\tt cw} \leqslant C_{\tt cov}\,, \forall p \geqslant 1\,.
    \end{equation*}
\end{lemma}
\cref{prop:cwcov} can be used for demonstrating a better regret bound in online RL when compared to that of $C_{\tt cov}$ as below.

We take the GOLF algorithm \cite{jin2021bellman} as an example to demonstrate the sample-efficient guarantees of online RL.
For self-completeness, we give a brief description on the GOLF algorithm \cite{jin2021bellman} in Algorithm~\ref{alg:golf}, see \cref{app:golf}.
This is a typical general function approximation algorithm in online RL, and yields sample-efficient guarantees if the BE dimension is small.
%The key step is line 7: optimization based exploration under the constraint of an identified confidence region $\mathcal{F}^{(t)}$ with a confidence parameter $\beta$. 
%The quantity $\mathcal{L}_h^{(t)}\left(f, f^{\prime}\right)$ can be regarded as an approximation of the squared Bellman error at step $h$.
Here we show that, under our coverage condition $C_{\tt cw}$, we can still achieve the sample-efficient guarantees for online RL, with the proof deferred to Appendix~\ref{app:thmgolfcw}.
\begin{proposition}\label{thm:golfcw}
    Under Assumptions~\ref{assum:rea} and \ref{assum:bellman}, there exists a constant $c$ and the data coverage coefficient $C_{\tt cw}$ in Definition~\ref{def:ccw} such that for any $\delta \in (0,1)$, if we choose $\beta = c \log \left( \frac{\mathscr{N}_{\mathcal{F}}(1/T) TH }{ \delta } \right)$ in the GOLF algorithm \ref{alg:golf}, with probability at least $1-\delta$, we have
    \begin{equation*}
        {\tt Regret} (T) \lesssim \mathcal{O} \left( H \sqrt{{C}_{\tt cw}^{\frac{1}{p}} \beta T \log T} \right)\,.
    \end{equation*}
\end{proposition}
{\bf Remark:} We obtain a better regret bound than \citep[Theorem 1]{xie2023role} due to an improved data coverage coefficient in Lemma~\ref{prop:cwcov}.
%\fh{when $p$ is large, we need to do less exploration? it's related to $C_{\tt cw} \mu_h^{\star} > 1$ or $< 1$.}

%This makes us to define a new ``exploration" phase based on $C_{\tt cw}$ to capture the earliest time at which $(s,a)$ has been explored sufficiently, leading to a different estimation.

Our result in \cref{thm:golfcw} demonstrates that if the coverage coefficient $C_{\tt cw}$ is small, the GOLF algorithm can achieve sublinear regret for sample-efficient guarantees without requiring the structure assumption of MDP. 
This is because, the coverage condition in fact implicitly imposes some structural assumptions on the MDP dynamics, see \citep[Theorem 4]{chen2019information} for details.
It is an intrinsic structural property of MDPs that describes the complexity of probability transitions.
This shares a similar spirit with the sub-optimality gap \cite{he2021logarithmic} on describing the complexity of MDPs under probability transitions.
Nevertheless, the condition of the sub-optimality gap is stronger because the reward feedback is also considered.

There appears a natural question on the relationship between coverage conditions and structural assumptions.
Since coverage conditions do not involve additional information, they are often weaker than structural assumptions.
For example, Sequential Exploration Coefficient (SEC) \cite{xie2023role}, as a structural assumption, is a general version of coverage condition, which admits
\begin{equation*}
    {\tt SEC} \lesssim C^{\frac{1}{p}}_{\tt cw} \log T \leqslant C_{\tt cov} \log T\,.
\end{equation*}
Apart from this, the relationship between various structural assumptions and coverage conditions requires a refined analysis but is
beyond the scope of this work.

Nevertheless, coverage conditions are still more general than linear MDP \cite{jin2020provably}.
For example, in the Atari game, the state space (raw pixels) can be very large, but the dynamics is determined by a small number of unobserved latent states.
This can be described as block MDP \cite{du2019provably}, and accordingly the 
coverability coefficient can be small as it scales only with the number of latent states instead of the size of the whole state space.

%{\bf Technical challenges:} The proof in this subsection is standard. employs the definition of $C_{\tt cw}$ and the relationship between  

%Based on our result in the $L_p$ space, we are ready to focus on the partial coverage condition, and then study under which coverage level, this condition is possible to ensure sample-efficient online RL. This is one main contribution of this work.

%Based on our result in the $L_p$ space, we are ready to disentangle the effect of partial/rest coverage condition on the sample-efficient online RL.

\subsection{Efficient online RL with exploratory offline data}
\label{sec:offline}
As mentioned before, the GOLF algorithm is not computation efficient due to the constructed ``global" confidence set.
To avoid sophisticated exploration schemes, one typical way is to use offline data for exploration, which is recently popular both empirically \cite{ball2023efficient} and theoretically \cite{song2022hybrid,wagenmaker2022leveraging}. 

Here we use the hybrid-Q algorithm \cite{song2022hybrid} to demonstrate the benefit of our data coverage condition when involving with offline data on the computation efficiency.
This algorithm is based on the classical fitted Q-iteration (FQI) algorithm \cite{ernst2005tree} and uses offline data regarding the distribution $\nu := \{ \nu_h \}_{h=1}^H$ for exploration, and thus the computation complexity of this algorithm is the same as FQI with a least squares regression oracle, refer to Appendix~\ref{app:hybridq} for details of the hybrid-Q algorithm.

Here we aim to demonstrate that without any structural assumption, the all-policy coverage conditions can ensure efficient online RL, both statistically and computationally if exploratory offline data are used.
In the following, we take the all-policy concentrability coefficient $C_{\pi}$ in Eq.~\eqref{eq:Cpi} and our coverage condition $C_{\tt cw}$ in Definition~\ref{def:ccw} as examples to illustrate this, with the proof deferred to Appendix~\ref{app:offccw}.

\begin{proposition}\label{prop:hyper}
    Under Assumptions~\ref{assum:rea} and \ref{assum:bellman}, then for any $\delta \in (0,1)$, $T \in \mathbb{N}$, if we choose $n_\mathrm{off} = T$ in Algorithm~\ref{alg:fqi} and denote $\beta := \log \left( \frac{\mathscr{N}_{\mathcal{F}}(1/T) TH }{ \delta } \right)$, with probability at least $1-\delta$,
\begin{case}
 under the all-policy concentrability coefficient $C_{\pi}$ in Eq.~\eqref{eq:Cpi}, we have
    \begin{equation*}
        {\tt Regret} \lesssim \mathcal{O} \left( C_{\pi} H \sqrt{\beta T} \right)\,.
    \end{equation*}
\end{case}
\begin{case}
    there exists a data distribution $\nu := \{ \nu_h \}_{h=1}^H$ that provides a single-policy concentrability, $\max_{s,a,h} \frac{ \mu_h^{\star}(s,a) }{ \nu^2_h(s,a) } < \widetilde{C}$,
    where $\mu_h^{\star}(s,a)$ realizes the value of the coverage coefficient $C_{\tt cw}$ endowed by $L^p(\mathrm{d}\mu)$ norm with $p \geqslant 1$ in Definition~\ref{def:ccw}, we have
     \begin{equation*}
        {\tt Regret} (T) \lesssim \mathcal{O} \left( {C}_{\tt cw}^{\frac{1}{p}} H \sqrt{ \beta \widetilde{C} T \log T} \right)\,.
    \end{equation*}
\end{case}
\end{proposition}
{\bf Remark:} We make the following remarks:\\        
\textit{i):} \citet{song2022hybrid} achieve the regret bound $\widetilde{\mathcal{O}}( C_{\pi^{\star}} H \sqrt{d \beta T} )$, where the single-policy concentrability coefficient $C_{\pi^{\star}}$ is defined in Eq.~\eqref{eq:Cpi}, and $d$ is the Bilinear rank or BE dimension.
Instead, in {\bf Case 1}, the structure assumptions on MDP are not needed to ensure the same regret if the all-policy concentrability coefficient $C_{\pi}$ is employed.\\
\textit{ii):} In {\bf Case 2}, if the all-policy concentrability coefficient ${C}_{\tt cw}^{\frac{1}{p}}$ is used, an extra single-policy concentrability coefficient $\widetilde{C}$ is needed.
As a single-policy version, it is often smaller than ${C}_{\tt cw}^{\frac{1}{p}}$, and can be even a constant if we take $\nu$ to match $\mu^{\star}$.
In this case, the $\widetilde{O}(\sqrt{T})$-regret can be still achieved without structural assumptions on MDP.

\AB{}
        {
                \AB{\cite{song2022hybrid}}
                {
                       single-policy coefficient $C_{\pi^{\star}}$ \\
                        structural assumptions
                }\\
                {\bf Case 1}: all-policy coefficient $C_{\pi}$ \\
                \AB{{\bf Case 2}:}
                {
                         all-policy coefficient ${C}_{\tt cw}^{\frac{1}{p}}$ \\
                        single-policy coefficient $\widetilde{C}$
                } \\
        }
        
{\bf Proofs techniques:}
To prove Propositions~\ref{thm:golfcw} and \ref{prop:hyper}, we follow the proof framework in \citep[Theorem 1]{xie2023role} on the regret analysis and the decomposition of the on-policy average Bellman error.
In \cref{thm:golfcw}, the difference lies in how to estimate the occupancy measure ratio by different coverage conditions. 
Further, in \cref{prop:hyper}, since no exploration scheme is used, we need to 
build the connection between $\rho_h^{(t)}$ and $\nu_h$ by coverage conditions, which is used for the estimation of the in-sample squared Bellman error.

\if 0
\begin{table*}[!htb]
	\centering
	\fontsize{9}{8}\selectfont
	\begin{threeparttable}
		\caption{Comparison with \cite{song2022hybrid} on the hybrid-Q algorithm with exploratory offline data}
		\label{tabbenefit}
		\begin{tabular}{ccccc}
			\toprule
			Results &coverage condition & structural assumption? & regret bound   \cr
			\midrule
			\cite{song2022hybrid} & $C_{\pi^{\star}}$ in Eq.~\eqref{eq:Cpi} & Bilinear class $d$ & $\widetilde{\mathcal{O}}(\sqrt{d C_{\pi^{\star}} T})$  \cr
                   \midrule
			\cref{prop:hyper} & $C_\pi$ in Eq.~\eqref{eq:Cpi}  & \xmark & $\widetilde{\mathcal{O}}(\sqrt{C_{\pi} T})$  \cr
                \midrule
			\cref{prop:hyper} & Def.~\ref{def:ccw}: $C_{\tt cw}$ in $L^{p}(\mathrm{d}\mu)$  & \xmark & $\widetilde{\mathcal{O}}\Big(C^{\frac{1}{p}}_{\tt cw} \sqrt{T}\Big)$  \cr
			\bottomrule
		\end{tabular}
	\end{threeparttable}
\end{table*}
\fi

The results in this warm-up section provide a good justification of the usage of general coverage conditions for efficient online RL.
This will motivate us to study partial/rest coverage trade-off and coverage conditions in linear MDP presented in the next two sections.
%Nevertheless, this part is not our main contribution in techniques but as a good starting point for understanding partial/rest coverage trade-off and coverage conditions in linear MDP presented in the next two sections.

\if 0
\begin{proposition}\label{prop:cpie}
    Under Assumptions~\ref{assum:rea} and \ref{assum:bellman}, there exists a data distribution $\nu := \{ \nu_h \}_{h=1}^H$ that admits a concentrability coefficient $C_{\pi}$ in Eq.~\eqref{eq:Cpi}.
    Then for any $\delta \in (0,1)$, $T \in \mathbb{N}$, if we choose $m_\mathrm{off} = T$ in Algorithm~\ref{alg:fqi}, with probability at least $1-\delta$, we have
    \begin{equation*}
        {\tt Regret} (T) \lesssim \mathcal{O} \left( \max_{\pi \in \Pi} {C}_{\pi} H \sqrt{ \beta T \log T} \right)\,.
    \end{equation*}
\end{proposition}
{\bf Remark:} Our result is based on all-policy concentrability which requires a highly exploratory dataset.
As suggested by \citep[Theorem 4]{chen2019information}, it also implicitly imposes structural assumptions on the MDP dynamics.
Accordingly, the following proposition demonstrates that we can get rid of the MDP structure assumption in \cite{song2022hybrid} if we use the same concentrability coefficient.
\fi

\section{Partial/rest coverage trade-off}
\label{sec:disen}

%In practice, we can not ensure that the offline data provide an all-policy coverage. 
As we know, partial coverage or even single coverage conditions are more realistic in practice, and widely studied in offline RL \cite{xie2021bellman,jin2021pessimism,zhan2022offline}.
However, \citet{xie2023role} point out that ${C_{\tt cov}}$ under a single-policy coverage can not ensure sample-efficient online RL.
Accordingly, in this section, based on our $L^p$ coverage concepts in Section~\ref{sec:lpcw}, we decouple the all-policy coverage condition into a partial-policy coverage condition by some underlying distribution and the rest-policy coverage condition, which is more realistic in practice.
Under this setting, we aim to diagnose the effect of partial/rest coverage condition on the regret bound.
%This allows for a better understanding and demonstrate a trade-off on partial/rest coverage conditions.

\subsection{Definition of partial/rest coverage condition}

Here we define the partial/rest coverage condition and then study the statistical guarantees of online RL algorithms.

{\bf Definition of partial policy class:} Motivated by $C_{\tt cov}$ in Definition~\ref{def:ccov} that can be regarded as a cumulative area over all possible $\rho_h^{\pi}$, we consider a possible policy class by evaluating how a policy is close to the \emph{reference} policy $\bar{\pi} := \{ \bar{\pi}_h \}_{h=1}^H$.
A nature metric is the total variation (TV) distance\footnote{Here the used metric between two distributions can be general, e.g., Wasserstein distance \cite{ruschendorf1985wasserstein,fournier2015rate}.}, and accordingly, the candidate policy set $\mathcal{M} = \{ \mathcal{M}_h \}_{h=1}^H$ is defined as
\begin{equation*}
    \mathcal{M}_h(\zeta) := \{ \pi_h: | \mathrm{TV} (\rho^{\pi_h}_h, \rho_h^{\bar{\pi}_h}) \leqslant \zeta  \}\,,
\end{equation*}
where the reference policy $\bar{\pi}$ can be set to the optimal policy $\pi^{\star}$ or any possible policy that is controlled by some (unknown) data distribution.
We can see that, $\bar{\pi}$ can be a high-quality policy or a low-quality policy, which is more realistic in practice.
%We can see that, at each horizon $h$, the policy class $\mathcal{M}_h(\zeta)$ contains all potential probability measures $\rho_h^{\pi}$ under a policy $\pi_h$ that are not far away from the occupancy measure under the optimal policy $\pi^{\star}_h$, controlled by a parameter $\zeta$.
Clearly we have $\zeta \in [0, 2]$ based on the definition of the TV distance.
If $\zeta = 0$, we only have single policy concentrability (i.e., only the reference policy) and if $\zeta = 2$, we can recover the whole policy class $\Pi$.
Hence this policy class $\mathcal{M}_h(\zeta)$ is a partial or incomplete policy class, and then the coverability coefficient defined over this policy class can be denoted as a partial coverage condition, introduced as below.

%Note that $\mathcal{M}_h$ covers some high-quality policies if $\zeta$ is small. In fact, we can also include some finite low-quality policies (which requires $\zeta$ to be large) into $\mathcal{M}_h$, which is more practical for offline RL. For ease of description, we include some low-quality policies such that $\sum_{(s,a)} \sup_{\pi \in \mathcal{M}} \rho_h^{\pi}(s,a) \geqslant 2$. This can be easily achieved as the occupancy measure generated by $\pi_h^{\star}$ and some low-quality policies can have few overlap. In this case, we can similarly define our coverage condition.

\begin{definition}
\label{def:ccovp}
The partial coverability coefficient ${P}_{\tt cov}(\zeta)$ is for a (partial) policy class $\mathcal{M}(\zeta)$
\begin{equation*}
\begin{split}
    P_{\tt cov}(\zeta) &:= \inf_{\mu_1,\ldots,\mu_H\in\Delta(\mathcal{S} \times \mathcal{A})}  \sup_{\pi \in \mathcal{M}(\zeta),h\in [H] }\,
  \left\|\frac{\rho_h^\pi}{\mu_h}\right\|_{\infty} \,.
\end{split}
\end{equation*}
\end{definition}
{\bf Remark:} For notional simplicity, we denote $ P_{\tt cov}(\zeta)$, $\mathcal{M}(\zeta)$ by $ P_{\tt cov}$, $\mathcal{M}$ for short.

Denote $ \hat{\mu}_h^\star := \argmin_{ \mu_h \subseteq \Delta(\mathcal{S}\times \mathcal{A} )} \sup_{\pi\in \mathcal{M}}\, \left\| \frac{\rho_h^\pi}{\mu_{h}}\right\|_{\infty}$, we can easily obtain the equivalent definition
$P_{\tt cov} =P_{\tt cr} := \max_{h\in [H] }  \sum_{(s,a) \in \mathcal{S} \times \mathcal{A}}\sup_{\pi \in \mathcal{M}} \rho_h^\pi(s,a) $, refer to the proof in Appendix~\ref{app:partialequ}.

Clearly $P_{\tt cov} \geqslant 1$, and the single policy concentrability implies $P_{\tt cov} = 1$ by taking $\mu_h:= \rho^{\bar{\pi}_h}_h$.
That means, $P_{\tt cov}$ looses the ability to represent the complexity of state transition in MDPs, and thus is insufficient to ensure sample-efficient learning in online RL.
In this case, we need to introduce extra conditions that aid for sufficient learning.

{\bf Coverage condition outside $\mathcal{M}$:} 
We give the definition of the rest coverage condition related to the policy $\hat{\mu}_h^\star$ over some policies outside $\mathcal{M}$, i.e., its complementary set $\bar{\mathcal{M}}$. 

\begin{definition}\label{def:ppar}
For any $(s,a) \in \mathcal{S} \times \mathcal{A}$, the state-action pair set ${\mathcal{B}^{\bar{\mathcal{M}}}} := \{ {\mathcal{B}}^{\bar{\mathcal{M}}}_h \}_{h=1}^H$ is denoted as
\begin{equation*}
    {\mathcal{B}}^{\bar{\mathcal{M}}}_h := \left\{ (s,a) \mid \rho_h^{\pi}(s,a) > c_1 P_{\tt cov} \hat{\mu}^{\star}_h(s,a), \forall \pi \in \bar{\mathcal{M}} \right\} \,,
\end{equation*}
with some constant $c_1 \geqslant 1$, then the partial coverage condition outside $\mathcal{M}$ is 
\begin{equation*}
    P_{\tt out}(\zeta) := \max_{h \in [H], \pi \notin \mathcal{M}} \left\| \frac{\rho_h^{\pi}}{\hat{\mu}^{\star}_h} \mathbbm{1}_{{\mathcal{B}}^{\bar{\mathcal{M}}}_h} \right\|_{L^2}^{\frac{1}{2}} \,,
\end{equation*}
defined in the $L^2$ space, and the indicator function $\mathbbm{1}_{{\mathcal{B}}^{\bar{\mathcal{M}}}_h} = 1$ if $(s,a) \in {\mathcal{B}}^{\bar{\mathcal{M}}}_h$, and otherwise is zero.
\end{definition}
{\bf Remark:} This quantity $P_{\tt out}$ defined in the $L^2$ space is also related to $\zeta$ due to $\mathcal{M}(\zeta)$, and we also omit it for notational simplicity. 

Clearly $P_{\tt out} \geqslant 0$, and there exists a trade-off between $P_{\tt cov}(\zeta)$ and $P_{\tt out}(\zeta)$ that depends on $\zeta$. If $\zeta$ increases, $P_{\tt cov}$ increases but $P_{\tt out}$ decreases. 
For example, if $\zeta = 2$, we have $P_{\tt cov} = C_{\tt cov}$ and $P_{\tt out} = 0$; if $\zeta = 0$, we have $P_{\tt cov} = 1$.
Accordingly, in this case $P_{\tt out}$ is used to measure the structural information of MDPs. Here we explain this a bit.

\if 0
\begin{definition}\label{def:pparstar}
Denote the state-action pair set $\widetilde{\mathcal{B}}^{(t)} := \{ {\mathcal{B}}^{(t)}_h \}_{h=1}^H$ with $t \in [T]$ and
\begin{equation*}
    \widetilde{\mathcal{B}}_h := \left\{ (s,a) \in \mathcal{S} \times \mathcal{A} \mid \rho_h^{(t)}(s,a) > c_1 {\rho}^{\pi^\star}_h(s,a) \right\} \,,
\end{equation*}
with some constant $c_1 \geqslant 1$, then the partial coverage condition related to $\pi^{\star}$ outside $\mathcal{M}$ is 
\begin{equation*}
    \widetilde{P}_{\tt par} := \max_{h \in [H]} \left\| \frac{\rho^{\pi}_h}{\rho^{\pi^\star}_h} \right\|_{\infty}, \quad  \forall (s,a) \in \bigcup_{t = \hat{\tau}_h}^T \widetilde{\mathcal{B}}_h^{(t)},~~ \pi \notin \mathcal{M} \,.
\end{equation*}
\end{definition}
\fi

In our proof (\emph{c.f.} Appendix~\ref{app:partialpcov}), we also show that if we take the reference policy $\bar{\pi}:= \pi^{\star}$ and $c_1$ large enough, e.g., $c_1 = \Omega(H)$, the probability that the optimal policy $\pi^{\star}$ visits this state-action pair set ${\mathcal{B}^{\bar{\mathcal{M}}}} $ is very small. 
That means, $P_{\tt out}$ can be regarded as the distribution shift between a policy $\pi$ and $\pi^{\star}$ on some low-probability set.
We can see, it describes the ability that an algorithm overcomes the difficult state-action pairs in MDPs, which can be also regarded as an instance-based metric.

%Note that $P_{\tt out}$ is a partial coverage condition because of the partial policy class $\pi \notin \mathcal{M}$ and state-action pairs ${\mathcal{B}}$ as a subset of $\mathcal{S} \times \mathcal{A}$.
%Accordingly, we have the following theorem, with the proof deferred to Appendix~\ref{app:partialpcov}.

\subsection{Sublinear regret bound}

Based on our definition on $P_{\tt cov}$ and $P_{\tt out}$, we have the following theorem that demonstrates how the partial/rest coverage condition affects the regret bound, with the proof deferred to Appendix~\ref{app:partialpcov}.
\begin{theorem}\label{thm:golfpcov}
    Under Assumptions~\ref{assum:rea} and \ref{assum:bellman}, there exists a constant $c_1$, the partial coverage coefficient $P_{\tt cov}$ in Definition~\ref{def:ccovp} and $P_{\tt out}$ in Definition~\ref{def:ppar}, then
    for any $\delta \in (0,1)$, $T \in \mathbb{N}$, if we choose $\beta = c \log \left( \frac{\mathscr{N}_{\mathcal{F}}(1/T) TH }{ \delta } \right)$ in GOLF, with probability at least $1-\delta$, we have
    \begin{equation*}
    \begin{split}
        {\tt Regret} & \lesssim \mathcal{O} \left( H\Big( \sqrt{c_1 P_{\tt cov} }  + \frac{ P_{\tt out} }{\sqrt{ P_{\tt cov} } } \Big)  \sqrt{\beta T \log T} \right)\,.
    \end{split}
    \end{equation*}
    Specifically, there always exists a proper $\zeta^{\star} \in [0,2]$ such that $P_{\tt out}(\zeta^{\star}) = \sqrt{c_1} P_{\tt cov}(\zeta^{\star})$, the above regret bound can be improved to
\begin{equation}\label{eq:pcovopt}
\begin{split}
    {\tt Regret} & \lesssim \mathcal{O} \left( H \sqrt{c_1^{1/2} \beta T P_{\tt out}(\zeta^{\star}) \log T } \right)\,.
    %& = \mathcal{O} \left( H \sqrt{c_1 \beta T P_{\tt out}(\zeta^{\star}) \log T } \right)
\end{split}
\end{equation}
which admits $P_{\tt out}(\zeta^{\star}) \leq C_{\tt cov}$. 
\end{theorem}
{\bf Remark:} One can choose $c_1$ to some constant up to $H$, so we remain $c_1$ in our bound. We make the following remarks.\\
\textit{i)}
If we only consider the single policy in $\mathcal{M}$, which implies $P_{\tt cov} = 1$, and our result is still applicable to ensure sample-efficient learning estimated by $P_{\tt out}$.
If we consider the whole policy class such that $P_{\tt cov} = C_{\tt cov}$ and then $P_{\tt out} = 0$, we can recover the result of \cite{xie2023role}.\\
\textit{ii)} Clearly, there exists a trade-off between $P_{\tt cov}(\zeta)$ and $P_{\tt out}(\zeta)$ that depends on $\zeta$.
That means, there always exists a proper $\zeta^{\star}$ such that Eq.~\eqref{eq:pcovopt} holds and $P_{\tt out}(\zeta^{\star}) \leqslant C_{\tt cov}$ by the property of the function $x + {c}/{x}$ for some constant $c$.
This demonstrates a better regret bound than \cite{xie2023role} by a good trade-off between $P_{\tt cov}$ and $P_{\tt out}$.
%Besides, it is non-easy to compare $1/P_{\tt avg}$ and $C^{\frac{1}{2}}_{\tt cov}$ in general cases.
%However, if the optimal policy is supported on all states \cite{dann2021beyond}, $P_{\tt avg}$ can be in a constant order; while $C_{\tt cov}$ is the cumulative reachability regarding all possible policies, and thus we can not ensure $C_{\tt cov}$ is small.
%\fh{besides, this paper also consider $\lim \inf_{T \rightarrow \infty} {\tt Regret}(T)$}

\cref{prop:hyper} extends the application scope of the hybrid-Q algorithm in the view of coverage conditions instead of structural assumptions, which provides a good justification on the study of coverage condition.

{\bf Proof sketch of \cref{thm:golfpcov}:} 
%In our proof, based on the regret analysis, the on-policy average Bellman error can be divided by two parts based on whether $(s,a)$ has been explored by $[P_{\tt cov} \hat{\mu}_h^{\star}]^p$. In the first part (the exploration phase), the on-policy average Bellman error can be directly controlled by $P_{\tt cov}$.
In our proof, the on-policy average Bellman error can be transformed to the occupancy measure ratio and the in-sample squared Bellman error.
The technical difficulty is to control the ratio when the on-policy occupancy measure is unbounded. The in-sample squared Bellman error can be directly estimated by \cite{jin2021bellman}.
To handle the ratio, we split the on-policy occupancy measure into two cases: 1) $\rho_h^{(t)}(s,a) \leqslant c_1 P_{\tt cov} \hat{\mu}^{\star}_h(s,a)$ and 2) $(s,a) \in {\mathcal{B}}^{\bar{\mathcal{M}}}$ which means $\rho_h^{(t)}$ is unbound by some (scaling) probability measure.
In the first case, it is upper bounded by $P_{\tt cov}\log T$;   
In the second case, $\rho_h^{(t)}$ cannot be controlled by previous occupancy measures $\{ \rho_h^{(i)} \}_{i=1}^{t-1}$ in terms of Bellman residual.
We build the connection between $\rho_h^{(t)}$ and $\hat{\mu}_h^{\star}$, introduce $P_{\tt out}$ to control such distribution shift, and trade-off $P_{\tt cov}$ and $P_{\tt out}$ for a better regret bound.

%\section{Benefits from coverage conditions}

%\fh{If this algorithm uses $\epsilon$-greedy, can we obtain some imporved results? e.g., few offline data?}

%\textit{ii):} $\mu_h^{\star}$ is not the optimal policy $\pi^{\star}$ and we also do not require the offline data distribution $\nu$ to cover $\pi^{\star}$. The choice of the offline data distribution is not limited to realize the value $C_{\tt cw}$. see the discussion in 

%aim to design a both computation and sample efficient RL algorithm. We require offline data under a good data coverage, which is beneficial to exploration in online reinforcement learning, by obtain computationally efficient algorithms.

\section{Coverage conditions help linear MDP}
\label{sec:linearmdp}
Till now we have already demonstrated that, without explicit structure assumptions on MDP, the new devised coverage conditions are able to ensure sample-efficient online RL in \cref{thm:golfcw} and~\cref{thm:golfpcov}, respectively.
By general coverage conditions, we are able to achieve better regret bound than \cite{xie2023role}.
Here we are also interested in
\begin{center}
    \emph{If the structural assumption is given, what can we still benefit from coverage conditions?}
\end{center}
In this section, we take the classical linear MDP using the LSVI-UCB algorithm \cite{jin2020provably} as an example, and demonstrate that a faster regret bound than $\widetilde{\mathcal{O}}(\sqrt{T})$ or even $\mathcal{O}(\log T)$ can be achieved if extra coverage conditions are employed.

%As we know, using the LSVI-UCB algorithm \cite{jin2020provably} allows for linear MDP to achieve $\widetilde{\mathcal{O}}(\sqrt{T})$ regret without any extra coverage condition. 
%In this case, we consider the following natural question
%\begin{center}
%\emph{If linear MDP implicitly satisfies some coverage conditions, can we achieve better sublinear regret than $\widetilde{\mathcal{O}}(\sqrt{T})$?}    
%\end{center}
%Here we give an affirmative answer to this question.
For ease of description, we give some notations here.
Details about the LSVI-UCB algorithm can be found in Appendix~\ref{app:lsvi}.
Denote the feature mapping $\phi(s,a) \in \mathbb{R}^d$ in linear MDP \cite{jin2020provably} satisfying $\| \phi(s,a)\|_2 \leqslant 1$, and 
$\Lambda_h^t$ constructed by the standard LSVI-UCB algorithm with the regularization parameter $\lambda$, i.e., 
    \begin{equation}\label{eq:lambdaht}
    \Lambda_h^t := \lambda I + \sum_{i=1}^{t-1} \phi(s_h^i,a_h^i) \phi(s_h^i,a_h^i)^{\!\top} \,.
\end{equation}

We assume that the underlying data distribution $\mu$ satisfies the following condition.

\begin{assumption}\label{assum:uef}
   \citep[feature coverage condition]{wang2021statistical} There exists a underlying distribution $\mu:=\{ \mu_h \}_{h=1}^H$ such that $\lambda_{\min}(\mathbb{E}_{\mu} [\phi(s,a) \phi(s,a)^{\!\top}]) \geqslant \gamma > 0$.
\end{assumption}
{\bf Remark:} We make the following remarks.\\
\textit{i)}: This assumption shares the similar spirit with the ``uniformly excited feature" assumption \cite{abbasi2019politex,papini2021reinforcement} but is much weaker than them as they require the minimum eigenvalue lower bounded under \emph{any} occupancy measure. Our assumption only requires the validity under the \emph{single} measure.\\ 
% The feature coverage condition is strictly weaker than concentrability
\textit{ii)}: This assumption can be easily achieved, e.g., $\mathbb{E} [\bm x \bm x^{\!\top}] \succ 0$ in statistics for linear feature mapping \cite{wainwright2019high}. Besides, another typical example is that, the minimum eigenvalue of neural tangent kernel \cite{jacot2018neural} can be lower bounded by a positive constant \cite{nguyen2021tight}.

\if 0
\begin{proposition}\label{prop:asylsvi}
       Consider linear MDP with $\| \phi(s,a)\|_2 \leqslant 1$, under Assumption~\ref{assum:uef} with $\gamma > 0$, using LSVI-UCB, with high probability, we have
    \begin{equation*}
  {\tt Regret}(T) \lesssim \widetilde{\mathcal{O}} \left(  \frac{dH^2}{\gamma} \sqrt{d \log T} \right)\,, \quad \mbox{when $T \rightarrow \infty$}\,.
\end{equation*}
\end{proposition}
{\bf Remark:} Without Assumption~\ref{assum:uef}, LSVI-UCB for linear MDP still achieves $\mathcal{O}(\sqrt{T})$ regret instead of $\mathcal{O}(\log T)$ regret even in the \emph{asymptotic} regime due to the estimation of $\sum_{h=1}^H \sum_{t=1}^T \| \phi_h^t \|_{(\Lambda_h^t)^{-1}}$.
\fi

%In RL theory, the \emph{non-asymptotic} regime is more common. We need to strength the condition on $\mu$ to obtain a faster rate than $\mathcal{O}(\sqrt{T})$ under the \emph{non-asymptotic} regime.
Based on our discussion, we can see Assumption~\ref{assum:uef} is much weaker than the ``uniformly excited feature" assumption and can be easily achieved in practice.
That means, this assumption might be not enough to ensure better results for linear MDP.
In this case, we need to strength the condition on the underlying distribution $\mu$ as below.
\begin{assumption}[low variance condition]\label{assum:lowv}
    For the LSVI-UCB algorithm with the empirical covariance matrix $\Lambda_h^t$ defined in Eq.~\eqref{eq:lambdaht}, there exists the underlying distribution $\mu = \{ \mu_h \}_{h=1}^H$ with $(s_h,a_h) \sim \mu_h$ such that 
    \begin{equation*}
        \mathbb{V} \left[\| \phi_h (s_h,a_h) \|^2_{(\Lambda_h^t)^{-1}} \right] \lesssim \frac{1}{\lambda^{2\alpha}}\,, \quad \alpha > 1\,.
    \end{equation*}
\end{assumption}
{\bf Remark:} We make the following remarks.\\
\textit{i)} Under the standard linear MDP setting, we always have the bounded random variable $\| \phi_h(s,a) \|^2_{(\Lambda_h^t)^{-1}} \leqslant 1/\lambda$, and thus its variance admits $\mathbb{V} [\| \phi_h (s_h,a_h) \|^2_{(\Lambda_h^t)^{-1}} ] \leqslant \frac{1}{4\lambda^{2}}$.
That means, Assumption~\ref{assum:lowv} always holds with $\alpha=1$ for any distribution.\\
\textit{ii)} Our assumption requires $\alpha > 1$, and in fact requires that the data distribution $\mu$ make $\| \phi_h (s_h,a_h) \|^2_{(\Lambda_h^t)^{-1}}$ concentrate around its mean, i.e., a low variance condition. 
It shares similar spirit with \cite{du2019provably,wang2021exponential} that characterizes the anti-concentration of a distribution $\mu$.\\
%This holds true if we take Bernoulli random variables with parameter $p_t \rightarrow 0$ \cite{scott2006minimax}.
%\textit{iii)} We give a concrete example here by denoting $\bm x:=(s_h,a_h)$ for short. Taking the feature mapping $\phi(\bm x) := \frac{\bm x}{\| \bm x\|_2} \exp(-|| \bm x ||^2_2)$ such that $|| \phi(\bm x) ||^2_2 \leq 1$. There exists a uniform distribution $\mu^*$ supported on $[\sqrt{\frac{\lambda}{d}}, c]^d$ for some constant $c$. In this case, we have $0 \leq \| \phi(\bm x) \|_2^2 \leq 1/\lambda$ with $\bm x \sim \mu^*$, and $|| \phi(\bm x) ||^2_{(\Lambda_h^t)^{-1}} \leq 1/\lambda^2$. In fact, the selection of $\mu^*$ can be quite general to ensure this.
%Accordingly, we have $\mathbb{V}[|| \phi(\bm x) ||^2_{(\Lambda_h^t)^{-1}}] \leq \frac{1}{4\lambda^4}$, which means $\alpha = 2$, and thus our assumption holds.
\textit{iii)} We give an example here by denoting $\bm x:=(s_h,a_h)$ for short, and the upper/lower bound of $\| \phi_h (\bm x) \|^2_{(\Lambda_h^t)^{-1}}$ as $M$ and $m$.
Accordingly, we have
\begin{equation}\label{eq:boundvar}
    \frac{1}{\lambda + t - 1} \leqslant m \leqslant \| \phi_h (\bm x) \|^2_{(\Lambda_h^t)^{-1}} \leqslant M \leqslant \frac{1}{\lambda}\,,
\end{equation}
by the Weyl inequality and $\bm a^{\!\top} \bm A \bm a \geqslant \lambda_{\min}(\bm A) \| \bm a \|_2^2$ for any PSD matrix $\bm A$.
Since Eq.~\eqref{eq:boundvar} holds for any distribution $\bm x \sim \mu$.
There exists some certain distributions $\mu$ such that the random variable $\| \phi_h (\bm x) \|^2_{(\Lambda_h^t)^{-1}}$ concentrates, i.e., $M-m$ is small.
For example, taking $M:=1/\lambda$, $m:= 1/\lambda - 1/\lambda^2$ such that $M- m \leq 1/\lambda^2$.
That means, under a certain distribution, the feature mapping $\phi_h$ has the similar (semi)-norm in the $(\Lambda_h^t)^{-1}$-(semi)-norm based space.
Then, by Popoviciu's inequality on variances, we have $\mathbb{V}[\| \phi_h (s_h,a_h) \|^2_{(\Lambda_h^t)^{-1}}] \leqslant \frac{1}{4(M-m)^2} \leqslant \frac{1}{4\lambda^4}$, which implies $\alpha = 2$, and thus our assumption holds.

Based on the above two assumptions, we are ready to improve the regret in linear MDP from $\widetilde{\mathcal{O}}(\sqrt{T})$ to faster rate and even in the logarithmic order by the following proposition, with the proof deferred to Appendix~\ref{app:lsvi}.
\begin{theorem}\label{thm:lsvi}
    For linear MDP using the LSVI-UCB algorithm, under Assumption~\ref{assum:uef} with $\gamma > 0$ and Assumption~\ref{assum:lowv} with $\alpha > 1$, taking the regularization parameter $\lambda := T^{\eta}$ with $ \eta \in (0,1]$ and the bonus parameter $\beta = \widetilde{\mathcal{O}} \left( \sqrt{\lambda} H(d+ \sqrt{\log \frac{1}{\delta}}) \right)$ for any $\delta \in (0,1)$, with probability at least $1 - \delta$, we have
    \begin{equation*}
    \begin{split}
   {\tt Regret}(T) & \lesssim  \left( \frac{H^2 d^2}{\gamma} \log T + \frac{H^2d \lambda \sigma}{\gamma} \sqrt{T} \right) \log \left( \frac{4}{ \delta} \right)  \\
    & = \left\{ \begin{array}{rcl}
				\begin{split}
					& \!\!  \mathcal{O} \left(\frac{d^2 H^2}{\gamma} \log T\right) ,~\mbox{if $\eta (\alpha - 1) \geqslant 1/2$} \\
					& \!\! \mathcal{O} \left( \frac{d H^2}{\gamma}  T^{\frac{1}{2}- \eta (\alpha - 1)} \right) ,~\mbox{if $\eta (\alpha \!-\! 1) \in (0,\frac{1}{2})$}\,.  \\
				\end{split}
			\end{array} \right. 
    \end{split}
\end{equation*}
\end{theorem}
{\bf Remark:} We make the following remarks.\\
\textit{i)} If we take $\alpha=1$, Assumption~\ref{assum:lowv} always holds.
Since Assumption~\ref{assum:uef} easily holds for a underlying distribution $\mu$, we can recover the $\widetilde{\mathcal{O}}(\sqrt{T})$-regret in \cite{jin2020provably}.
\\
\textit{ii)} If $ \eta (\alpha - 1) \geqslant 1/2 $, that means, $\alpha$ can be large, the regret enjoys the logarithmic order of $T$. If $ 0< \eta (\alpha - 1) < 1/2 $, we have a sublinear $\mathcal{O}(T^{\frac{1}{2}- \eta (\alpha - 1)})$ regret, faster than the classical $\mathcal{O}(\sqrt{T})$ regret. \\
\textit{iii)} The regularization parameter $\lambda$ decreases with the increasing $T$ though we use $\lambda := \mathcal{O}(T^{\eta})$ with $\eta \in (0,1]$.
This is because, the ``true" regularization parameter is $\lambda/T$ as we need to scale LSVI with the number of the involved state-action pairs. 
The regularization parameter decaying with the number of samples is fair and commonly used in learning theory \cite{cucker2007learning}.
Besides, taking $\eta = 0$ in the regularization parameter $\lambda$ is able to improve the regret rate under a slight changes of Assumption~\ref{assum:lowv}.
Detailed discussion can be found in Appendix~\ref{app:regu}. \\
\textit{iv)} Instance-dependent regret bound has been widely studied for linear MDP with the logarithmic-order regret \cite{he2021logarithmic} under the minimal sub-optimality gap (strictly larger than zero) and further improved to the constant regret \cite{papini2021reinforcement}.
This requires a separation between the optimal action and the rest ones; while our assumptions focus on a ``distinct" feature mapping under certain distributions.
%Besides, we also need to remark that, 
%only changing $\lambda$ cannot improve the regret bound in classical LSVI-UCB for linear MDP.

\if 0
\begin{table}[h]
	\centering
	\fontsize{9}{8}\selectfont
	\begin{threeparttable}
		\caption{Comparison with \cite{song2022hybrid} on the hybrid-Q algorithm with exploratory offline data}
		\label{tabbenefit}
		\begin{tabular}{ccccc}
			\toprule
			Results & Assumption & Regret   \cr
			\midrule
			\citet{he2021logarithmic} &  minimal sub-optimality gap  & logarithmic   \cr
                   \midrule
			\cref{prop:hyper} & \xmark & $\widetilde{\mathcal{O}}(\sqrt{C_{\pi} T})$  \cr
                \midrule
			\cref{prop:hyper} & \xmark & $\widetilde{\mathcal{O}}\Big(C^{\frac{1}{p}}_{\tt cw} \sqrt{T}\Big)$  \cr
			\bottomrule
		\end{tabular}
	\end{threeparttable}
\end{table}
\fi

{\bf Proof sketch:} We provide a new proof framework on LSVI-UCB for linear MDPs to achieve faster regret bound.
By a telescoping lemma \cite{jiang2022notes}, the regret can be upper bounded by $\| \phi_h (s_h,a_h) \|^2_{(\Lambda_h^t)^{-1}}$ over the on-policy measure $\rho_h^{(t)}$. 
The key challenge is, if we directly apply change-of-measure: from $\rho_h^{(t)}$ to the underlying distribution $\mu_h$, the elliptical potential lemma is invalid. In this case, in our analysis, we build the connection between $\mathbb{E}_{\rho_h^{(t)}}\| \phi_h (s_h,a_h) \|^2_{(\Lambda_h^t)^{-1}} $ and $\mathbb{E}_{\mu_h}\| \phi_h (s_h,a_h) \|^2_{(\Lambda_h^t)^{-1}} $ by our coverage condition in Assumption~\ref{assum:uef}.
Accordingly, the regret can be bounded by $\mathbb{E}_{\mu_h}\| \phi_h (s_h,a_h) \|^2_{(\Lambda_h^t)^{-1}} $ and thus improved if $\mu$ has a lower variance in Assumption~\ref{assum:lowv}.

%\input{sections/experiment}
% \vspace{-0.2cm}
\section{Conclusion}
\label{sec:conclusion}
 %\vspace{-0.15cm}

Our work focuses on an interesting question for efficient online RL: \emph{what can online RL benefit from coverage conditions?} 
In our setting, the standard structural assumptions on MDPs are substituted by coverage conditions in online RL.
We answer this question in three folds: sample efficient guarantees of GOLF by various coverage conditions, the sample- and computation- efficiency guarantees of hybrid-Q, and faster regret bound of LSVI-UCB in linear MDP.  
Our results provide more possibilities of digging the potential and the utility of various coverage conditions.
We believe that the relationship between coverage conditions and structural assumptions is always an interesting and important direction in general function approximation in RL, both empirically and theoretically, which requires more refined analysis in the future.

\if 0
We focus on function approximation in online RL based on various coverage conditions instead of classical structural assumptions on MDPs, and demonstrate the benefits from three folds:
1) Various coverage conditions based on the $L^p$ variant of concentrability, the density ratio realizability, trade-off between partial/rest coverage condition can ensure sample-efficient online RL;
2) structural assumptions in the hybrid-Q algorithm can be substituted by our coverage conditions, and hybrid-Q under this setting still achieves both sample- and computation- efficiency for online RL;
3) LSVI-UCB is able to achieve faster regret bound than $\widetilde{\mathcal{O}}(\sqrt{T})$ on linear MDP if the underlying distribution admits a low-variance condition.
\fi

\section*{Acknowledgement}
The authors would like to thank anonymous reviewers for their constructive suggestions.

Fanghui is supported by SNF project – Deep Optimisation of the Swiss National Science Foundation (SNSF) under grant number 200021\_205011; Luca is funded in part through a PhD fellowship of the Swiss Data Science Center, a joint venture between EPFL and ETH Zurich; 
Volkan is supported by the European Research Council (ERC) under the European Union's Horizon 2020 research and innovation programme (grant agreement n°725594 - time-data).

\bibliography{refs}
\bibliographystyle{abbrvnat}

%\newpage

%\appendix
%\onecolumn
%\allowdisplaybreaks

\newpage
\appendix
\onecolumn

\begin{center}
\vspace{7pt}
{\Large \fontseries{bx}\selectfont Appendix}
\end{center}

\renewcommand{\contentsname}{Table of Contents}
\etocdepthtag.toc{mtappendix}
\etocsettagdepth{mtchapter}{none}
\etocsettagdepth{mtappendix}{subsection}
\tableofcontents

\newpage

\if 0
The Appendix is organized as follows:
\begin{itemize}
    \item In ~\cref{app:golf}, we include the GOLF \cite{jin2021bellman} in Algorithm~\ref{alg:golf} for self-completeness.
    \item In \cref{app:seclp}, we present the proof under our $L^p$ coverage condition in Section~\ref{sec:lpcw}.
    \item In~\cref{sec:finitely_width}, we extend the results of infinitely width to finite-width and provide the proof for them.
    \item In~\cref{sec:Relationship_NTK_Generalization}, we prove~\cref{thm:NTK_Generalization}.
    \item In~\cref{sec:discussion}, we discussion some key points of the proof and the motivation of the analysis.
    \item In \cref{sec:additional_experiments}, we detail our experimental settings, our Eigen-NAS algorithm as used in~\cref{ssec:NAS_201_experiment}. We conduct additional numerical validations.
    \item Finally, in~\cref{sec:nas_societal_impact}, we discuss the societal impact of this work.
\end{itemize}
\fi

\section{Flowchart of GOLF}
\label{app:golf}

For self-completeness, we include the flowchart of GOLF \cite{jin2021bellman} in Algorithm~\ref{alg:golf} here.
This is a typical general function approximation algorithm in online RL, and yields sample-efficient guarantees if the BE dimension is small.
The key step is line 7: optimization based exploration under the constraint of an identified confidence region $\mathcal{F}^{(t)}$ with a confidence parameter $\beta$. 
The quantity $\mathcal{L}_h^{(t)}\left(f, f^{\prime}\right)$ can be regarded as an approximation of the squared Bellman error at step $h$.

\begin{algorithm}[h] 
\caption{GOLF \cite{jin2021bellman}}
\begin{algorithmic}[1] 
\STATE{\textbf{Input:} Function class: $\mathcal{F}$, confidence parameter $\beta$}
\STATE Initialize $\mathcal{F}^{(0)} \leftarrow \mathcal{F}$, $\mathcal{D}^{(0)}_h = \emptyset\;\;\forall h \in [H]$
\FOR{$t = 1, \dots, T$}  
\STATE Let $\pi^{t}$ be the greedy policy w.r.t. \(f^{t}\) i.e., $f^{t} = \argmax_{f \in \mathcal{F}^{(t-1)}} f(s_1, \pi_{f,1}(s_1))$. 
\STATE{For each $h \in [H]$, execute $\pi^{t}$ and obtain a trajectory $\{  (s_h^{t}, a_h^{t}, r_h^{t}) \}_{h=1}^H$ }
\STATE{For each $h \in [H]$, dataset augment: $\mathcal{D}_h^{(t)} \leftarrow \mathcal{D}_h^{(t-1)} \bigcup \{ (s_h^{t}, a_h^{t}, r_h^{t}, s_{h+1}^{t}) \}$. }
\STATE{Update the confidence set with $f_{H+1}=0$:
  \begin{equation*}
  \begin{split}
      \mathcal{F}^{(t)} \leftarrow & \bigg\{f \in \mathcal{F}: \mathcal{L}_h^{(t)}\left(f_h, f_{h+1}\right) - \min _{f_h^{\prime} \in \mathcal{F}_h} \mathcal{L}_h^{(t)}\left(f_h^{\prime}, f_{h+1}\right) \leq \beta, \quad \forall h \in[H]\bigg\} 
      \end{split}
  \end{equation*}
where $\mathcal{L}_h^{(t)}\left(f, f^{\prime}\right):=\sum_{(s, a, r, s^{\prime}) \in \mathcal{D}_h^{(t)}}\Big[ f(s, a)-r-\max _{a^{\prime} \in \mathcal{A}} f^{\prime}\left(s^{\prime}, a^{\prime}\right)\Big]^2, \forall f, f^{\prime} \in \mathcal{F}$.
}
\ENDFOR
\STATE{\textbf{Output:} a policy uniformly sampled from $\{ \pi^{t} \}_{t=1}^T$} 
\end{algorithmic}\label{alg:golf} 
\end{algorithm}

\section{Proofs for Section~\ref{sec:cwrl}} 
\label{app:seclp}
In this section, we provide the proofs in Section~\ref{sec:lpcw} that the $L^p$ coverage conditions are identified to ensure sample efficient online RL.
\cref{app:cwcov} gives the proof of \cref{prop:cwcov}, and the proof of \cref{thm:golfcw} can be found in \cref{app:thmgolfcw}.

\subsection{Proof of \cref{prop:cwcov}}
\label{app:cwcov}
\begin{proof}
According to Definition~\ref{def:ccw}, the formulation of $C_{\tt cw}$ endowed by the $L^p(\mathrm{d}\mu_h)$ norm implies
\begin{equation}\label{eq:ccwcov}
    \begin{split}
        C_{\tt cw} & := \inf_{\mu_1, \cdots, \mu_H \in \Delta(\mathcal{S} \times \mathcal{A})} \sup_{\pi \in \Pi, h \in [H]}\,
  \left\| \frac{\rho_h^\pi}{\mu_h} \right\|^p_{L^p(\mathrm{d}\mu_h)} \leqslant \inf_{\mu_1, \cdots, \mu_H \in \Delta(\mathcal{S} \times \mathcal{A})} \sum_{(s,a)} \sup_{\pi \in \Pi, h \in [H]} \frac{[ \rho_h^{\pi}(s,a)]^p}{[\mu_h(s,a)]^{p-1}} \\
  & :=  \inf_{\mu_{\tilde{h}} \in \Delta(\mathcal{S} \times \mathcal{A})}  \sum_{(s,a)} \sup_{\pi \in \Pi} \frac{ [\rho_{\tilde{h}}^{\pi}(s,a)]^p}{[\mu_{\tilde{h}}(s,a)]^{p-1}} \quad \mbox{for a certain $\tilde{h} \in [H]$}\\
  & =  \inf_{\mu_{\tilde{h}} \in \Delta(\mathcal{S} \times \mathcal{A})}  \sum_{(s,a)} \left( \frac{\sup_{\pi} \rho_{\tilde{h}}^{\pi}(s,a)}{\mu_{\tilde{h}}(s,a)} \right)^{p-1} \sup_{\pi} \rho_{\tilde{h}}^{\pi}(s,a) \\
  & \leqslant \inf_{\mu_{\tilde{h}} \in \Delta(\mathcal{S} \times \mathcal{A})}  \left(   \max_{(s,a)} \frac{\sup_{\pi} \rho_{\tilde{h}}^{\pi}(s,a)}{\mu_{\tilde{h}}(s,a)} \right)^{p-1} \sum_{(s,a)} \sup_{\pi} \rho_{\tilde{h}}^{\pi}(s,a) \,,
  \end{split}
\end{equation}
where the first inequality holds by Jensen inequality for a convex function $\sup$.

Based on the formulation of $C_{\tt cov}$ in Definition~\ref{def:ccov}, we have
\begin{equation*}
    C_{\tt cov} =  \max_{h\in [H] }  \sum_{(s,a) \in \mathcal{S} \times \mathcal{A}}\sup_{\pi \in \Pi} \rho_h^\pi(s,a) \geqslant  \sum_{(s,a) \in \mathcal{S} \times \mathcal{A}}\sup_{\pi \in \Pi} \rho_{\tilde{h}}^\pi(s,a) \,,
\end{equation*}
which implies
\begin{equation*}
    \begin{split}
        C_{\tt cw} 
  & \leqslant \inf_{\mu_{\tilde{h}} \in \Delta(\mathcal{S} \times \mathcal{A})}  \left(   \max_{(s,a)} \frac{\sup_{\pi} \rho_{\tilde{h}}^{\pi}(s,a)}{\mu_{\tilde{h}}(s,a)} \right)^{p-1} C_{\tt cov} \\
  & = C_{\tt cov} \left( \inf_{\mu_{\tilde{h}} \in \Delta(\mathcal{S} \times \mathcal{A})}  \sup_{\pi \in \Pi}  \left\| \frac{ \rho_{\tilde{h}}^{\pi}}{\mu_{\tilde{h}}} \right\|_{\infty} \right)^{p-1} \\
  & \leqslant C_{\tt cov}^p\,,
    \end{split}
\end{equation*}
where the second equality uses Definition~\ref{def:ccov} for $C_{\tt cov}$ and the involved functions are monotonic w.r.t $p$. Finally we finish the proof.
\end{proof}

\subsection{Proof of \cref{thm:golfcw}}
\label{app:thmgolfcw}

Our proof framework follows \citep[Theorem 1]{xie2023role}, and there is only one slight difference involved with the weaker data coverage coefficient $C_{\tt cw}$, which leads to a different ``exploration" phase based on $C_{\tt cw}$.
For self-completeness, we present the detailed proof here, which is also helpful to our remaining results.
%to capture the earliest time at which $(s,a)$ has been explored sufficiently, leading to a different estimation.

\begin{proof}[Proof of \cref{thm:golfcw}]

For every step $h$, denote
\begin{equation*}
    \mu_h^\star := \argmin_{ \mu_h \subseteq \Delta(\mathcal{S}\times \mathcal{A} )} \sup_{\pi\in\Pi}\, \left\| \frac{\rho_h^\pi}{\mu_{h}}\right\|^p_{L^p(\mathrm{d}\mu_{h})}\,,
\end{equation*}
we have
\begin{equation}\label{eq:ceffstar}
    C_{\tt cw} = \sup_{\pi \in \Pi, h \in [H]} \sum_{(s,a)} \frac{[\rho_h^{\pi}(s,a)]^p}{[\mu_h^{\star}(s,a)]^{p-1}} \geqslant \frac{[\rho_h^{(t)}(s,a)]^p}{[\mu_h^{\star}(s,a)]^{p-1}}\,, \forall t, h, (s,a) \,.
\end{equation}

For notational simplicity, we adopt the shorthand $\rho_h^{(t)} := \rho_h^{\pi^{(t)}}$, and define
\begin{equation*}
\label{eq:def_dbar}
  \tilde{\rho}_h^{(t)} (s,a) := \sum_{i = 1}^{t - 1} \rho_h^{(i)} (s,a)\,.
\end{equation*}
which is the summation of all previous occupancy measure before episode $t$.
Note that $\tilde{\rho}_h^{(t)}$ is not a probability measure because it is unnormalized.
Accordingly, we introduce the notion of an ``exploration'' phase for each state-action pair $(s,a)\in \mathcal{S} \times \mathcal{A}$ based on $C_{\tt cw}\mu^{\star}_h$ such that
\begin{equation}\label{eq:tauh}
    \tau_h(s,a) = \min\left\{t \mid \tilde{\rho}_h^{(t)}(s,a) \geqslant [C_{\tt cw}\mu^{\star}_h(s,a)]^p \right\}\,,
\end{equation}
which describes the earliest time at which $(s,a)$ has been explored.
We refer to $t < \tau_h(s,a)$ as the exploration phase for $(s,a)$.

In the next, following \cite{xie2023role} on the regret decomposition, denoting $\delta_h^{(t)}(s,a) := f_h^{(t)}(s,a)-(\mathcal{T}_h f_{h+1}^{(t)})(s,a)$, we have
\begin{equation*}
    \begin{split}
          {\tt Regret} & \leqslant \sum_{t = 1}^{T} \left( f_1^{(t)}(s_1,\pi_{f^{(t)}_1,1}(s_1)) -J(\pi^{(t)}) \right) 
  = \sum_{t = 1}^{T} \sum_{h=1}^{H} \mathbb{E}_{(s,a)\sim{}\rho_h^{(t)}}\big[f_h^{(t)}(s,a)-(\mathcal{T}_h f_{h+1}^{(t)})(s,a) \big]  \\
  & = \sum_{t = 1}^{T} \sum_{h=1}^{H} \mathbb{E}_{(s,a)\sim{}\rho_h^{(t)}} \left[ \delta_h^{(t)}(s,a) \mathbbm{1}[t < \tau_h(s,a)] \right] + \sum_{t = 1}^{T} \sum_{h=1}^{H} \mathbb{E}_{(s,a)\sim{}\rho_h^{(t)}} \left[ \delta_h^{(t)}(s,a) \mathbbm{1}[t \geqslant \tau_h(s,a)] \right] \,,
    \end{split}
\end{equation*}
where the first term is the ``exploration" phase and the second term is the stable phase.
%To proceed, we use a change of measure argument to relate the on-policy \emph{average} Bellman error $\mathbb{E}_{(s,a)\sim{}\rho_h^{(t)}}[\delta_h^{(t)}(s,a)]$ appearing above to the in-sample \emph{squared} Bellman error $\mathbb{E}_{(s,a)\sim{}\tilde{\rho}_h^{(t)}}[\delta_h^{(t)}(s,a)^2]$; the latter is small.

In particular, for the ``exploration" phase, we use $|\delta_h^{(t)}| \leqslant 1$ to bound
\begin{equation*}
\begin{split}
  \sum_{t = 1}^{T} \mathbb{E}_{(s,a)\sim{}\rho_h^{(t)}}\left[\delta_h^{(t)}(s,a)\mathbbm{1}[t < \tau_h(s,a)]\right] & \leqslant
  \sum_{(s,a)}\sum_{t<\tau_h(s,a)} \rho_h^{(t)}(s,a)=
  \sum_{(s,a)}\tilde{\rho}_h^{(\tau_h(s,a))}(s,a) \\
  & = \sum_{(s,a)} [\tilde{\rho}_h^{(\tau_h(s,a) - 1)}(s,a) + \rho_h^{(\tau_h(s,a) - 1)}(s,a)] \\
  & \leqslant \sum_{(s,a)} [C_{\tt cw} \mu^\star_h(s,a)]^p + \sum_{(s,a)} C_{\tt cw}^{\frac{1}{p}} [\mu^\star_h(s,a)]^{\frac{p-1}{p}} \\
  & \leqslant C_{\tt cw}^p + C_{\tt cw}^{\frac{1}{p}} \sum_{s,a} [\mu^\star_h(s,a)]^{\frac{p-1}{p}} \\
  & \lesssim C_{\tt cw}^p \,,
  \end{split}
\end{equation*}
where the second inequality holds by Eqs.~\eqref{eq:ceffstar},~\eqref{eq:tauh}, and the last inequality holds by
\begin{equation}\label{eq:boundedp}
    \sum_{(s,a)} [\mu^\star_h(s,a)]^{\frac{p-1}{p}} \leqslant \sum_{(s,a)} \sqrt{\mu^\star_h(s,a)} < C \,,
\end{equation}
for some constant $C$.

For the stable phase, by change-of-measure, we have
\begin{align}
\nonumber
&~ \sum_{t = 1}^{T} \mathbb{E}_{(s,a)\sim{}\rho_h^{(t)}}\left[\delta_h^{(t)}(s,a)\mathbbm{1}[t \geqslant \tau_h(s,a)]\right]
\\
\nonumber
&= ~ \sum_{t = 1}^{T} \sum_{(s,a)} \rho_h^{(t)}(s,a) \left( \frac{\tilde{\rho}_h^{(t)}(s,a)}{\tilde{\rho}_h^{(t)}(s,a)} \right)^{\frac{1}{2}} \delta_h^{(t)}(s,a) \mathbbm{1}[t \geqslant \tau_h(s,a)]
\\
\label{eq:reg_CS}
& \leqslant \sqrt{\underbrace{\sum_{t = 1}^{T} \sum_{(s,a)} \frac{\left( \mathbbm{1}[t \geqslant \tau_h(s,a)] \rho_h^{(t)}(s,a) \right)^2}{\tilde{\rho}_h^{(t)}(s,a)} }_{:= {\tt I_A}}} \cdot  \sqrt{\underbrace{\sum_{t = 1}^{T} \sum_{(s,a)} \tilde{\rho}_h^{(t)}(s,a) \left(\delta_h^{(t)}(s,a)\right)^2\mathbbm{1}[t \geqslant \tau_h(s,a)]}_{:= {\tt I_B}}},
\end{align}
where the last inequality is an application of Cauchy-Schwarz inequality. 

We bound the first term ${\tt I_A}$ in Eq.~\eqref{eq:reg_CS} with
\begin{equation}\label{eq:term1}
\begin{split}
    {\tt I_A} := \sum_{t = 1}^{T} \sum_{(s,a)} \frac{\left( \mathbbm{1}_{ \{t \geqslant \tau_h(s,a) \} } \rho_h^{(t)}(s,a) \right)^2}{\tilde{\rho}_h^{(t)}(s,a)} & \leqslant 2 \sum_{t = 1}^{T} \sum_{(s,a)} \frac{\left( \mathbbm{1}_{ \{t \geqslant \tau_h(s,a) \} } \rho_h^{(t)}(s,a) \right)^2}{[C_{\tt cw} \mu^{\star}_h(s,a)]^p + \tilde{\rho}_h^{(t)}(s,a)} \\
    & \lesssim \sum_{t = 1}^{T} \sum_{(s,a)} \rho_h^{(t)}(s,a) \frac{ \rho_h^{(t)}(s,a) }{[C_{\tt cw} \mu^{\star}_h(s,a)]^p + \tilde{\rho}_h^{(t)}(s,a)} \\
    & \leqslant \sum_{t = 1}^{T} \sum_{(s,a)} C_{\tt cw}^{\frac{1}{p}} [\mu^{\star}_h(s,a)]^{\frac{p-1}{p}} \frac{ \rho_h^{(t)}(s,a) }{[C_{\tt cw} \mu^{\star}_h(s,a)]^p + \tilde{\rho}_h^{(t)}(s,a)} \\
    & \lesssim C_{\tt cw}^{\frac{1}{p}} \sum_{(s,a)} [\mu^{\star}_h(s,a)]^{\frac{p-1}{p}} \log T \quad \mbox{[using \cref{lem:per_sa_ep}]} \\
    & \lesssim C_{\tt cw}^{\frac{1}{p}} \log T \quad \mbox{[using Eq.~\eqref{eq:boundedp}]}\,,
\end{split}
\end{equation}
where the first inequality uses $\widetilde{\rho}_h^{(t)}(s,a) \geqslant \frac{1}{2}\widetilde{\rho}_h^{(t)}(s,a) + \frac{1}{2}[C_{\tt cw} \mu^{\star}(s,a)]^p$ and the third inequality holds by Eq.~\eqref{eq:ceffstar}.

For the second term ${\tt I_B}$ in Eq.~\eqref{eq:reg_CS}, we can directly employ the result of \cite{jin2021bellman}, see \cref{lem:jin2021}.
By taking $\beta = c \log \left( \frac{\mathscr{N}_{\mathcal{F}}(1/T) TH }{ \delta } \right)$ for some constant $c$ and $\delta \in (0,1)$, the quantity ${\tt I_B}$ holds with probability at least $1 - \delta$
\begin{equation*}
    {\tt I_B} \lesssim \mathcal{O}(\beta T)\,.
\end{equation*}
Combining the results of  the ``exploration" phase and the stable phase, our regret bound holds with probability at least $1 - \delta$
\begin{equation*}
   {\tt Regret} \leqslant \sum_{t=1}^T \sum_{h=1}^H \mathbb{E}_{(s,a)\sim{}\rho_h^{(t)}} [\delta_h^{(t)}(s,a)] \lesssim \mathcal{O} \left(HC_{\tt cw}^p + H \sqrt{C_{\tt cw}^{\frac{1}{p}} \beta T \log T} \right) = \mathcal{O} \left( H \sqrt{C_{\tt cw}^{\frac{1}{p}} \beta T \log T} \right)\,,
\end{equation*}
which concludes the proof.
\end{proof}

\section{Proof for Section~\ref{sec:offline}}
\label{app:offline}

In this section, we firstly include the flowchart of the hybrid-Q algorithm in \cref{app:hybridq} for self-completeness, and then present the proof of \cref{prop:hyper} in \cref{app:offccw}.

\subsection{Flowchart of the hybrid-Q algorithm}
\label{app:hybridq}

We include the flowchart of hybrid-Q in Algorithm~\ref{alg:fqi} here for self-completeness.
The idea of the hybrid-Q algorithm is intuitive.
It is based on the classical fitted Q-iteration (FQI) algorithm on the offline dataset $\mathcal{D}_h^{\nu}$ and on-policy trajectory generated by the current policy interacting with the environment.
This algorithm avoids sophisticated exploration schemes in online RL but uses offline data for exploration, and thus the computation complexity of this algorithm is the same as FQI with a least square regression oracle.

\begin{algorithm}[t] 
\caption{The Hybrid-Q algorithm using both offline and online data \cite{song2022hybrid}}
\begin{algorithmic}[1] 
\STATE{\textbf{Input:} Value function class: $\mathcal{F}$, offline dataset \(\mathcal{D}^{\nu}_h\) of size \(n_\mathrm{off}\) for \(h \in [H-1]\)}
\STATE Initialize \(f_h^1(s, a) = 0\). 
\FOR{episode $t = 1, \dots, T$}  
\STATE Let $\pi^t$ be the greedy policy w.r.t. \(f^t\) i.e., $\pi_h^t(s) = \argmax_a f^t_h(s,a).$ 
\STATE{For each $h$, sample $s_h \sim \rho_h^{\pi^t}$, $a_h \sim \pi^{t}(\cdot|s_h,a_h)$, and  $\mathcal{D}_h^{(t)} \leftarrow \mathcal{D}_h^{(t-1)} \bigcup \{ (s_h^{t}, a_h^{t}, r_h^{t}, s_{h+1}^{t}) \}$.   {// Online collection} }
 \vspace{1mm} 
%{// FQI using both online and offline data}
 \vspace{1mm}
 \label{line:online_sample}
\STATE Set $f_H^{t+1}(s,a) = 0$.  \\ 
\FOR{$h = H-1, \dots, 0$}  \label{line:fqi_iteration}
\STATE Estimate \(f_h^{t+1}\) using FQI on both offline and online data by defining $\varrho_h^t := [f(s, a)-r-\max _{a^{\prime} \in \mathcal{A}} f_{h+1}^{t+1} \left(s^{\prime}, a^{\prime}\right) ]^2$: 
\begin{equation*}\label{eq:alg_regression} 
\begin{split}
 f_{h}^{t+1} \! \leftarrow \! \argmin_{f\in \mathcal{F}_h}   \bigg\{ 
 \sum_{(s,a,r,s') \in \mathcal{D}^\mu_{h}} \!\! \varrho_h^t \! + \!  \sum_{\left(s, a, r, s^{\prime}\right) \in \mathcal{D}_h^{(t)}} \!\varrho_h^t \bigg\}\,.
\end{split}
\end{equation*} 
\ENDFOR
\ENDFOR
\end{algorithmic}\label{alg:fqi} 
\end{algorithm} 

\subsection{Proof of \cref{prop:hyper}}
\label{app:offccw}
%For the hybrid-Q algorithm in \cref{alg:fqi}, which is both statistically and computationally efficient based on offline data, here we present the proof to bound it.

In this section, we aim to prove that, using $C_{\pi}$ in Eq.~\eqref{eq:Cpi} or $C_{\tt cw}$  is able to ensure Algorithm~\ref{alg:fqi} statistically and computationally efficient.

\begin{proof}[Proof of \cref{prop:hyper} ]

We firstly prove {\bf Case 1} and then {\bf Case 2}.

{\bf Proof of Case 1:}

Lemma~\ref{lem:fqioff} implies that, for any $\delta \in (0,1)$, by taking $\beta = c \log \left( \frac{\mathscr{N}_{\mathcal{F}}(1/T) TH }{ \delta } \right)$ for some constant $c$, with probability at least $1 - \delta$, we have
\begin{equation*}
    \sum_{t=1}^T \mathbb{E}_{\nu_h} [\delta_h^{(t)}(s,a)]^2 \lesssim \frac{ \beta T}{n_\mathrm{off}}\,.
\end{equation*}
Accordingly, we have
\begin{equation*}
    \begin{split}
        \sum_{t = 1}^{T} \mathbb{E}_{(s,a)\sim{}\rho_h^{(t)}}\left[\delta_h^{(t)}(s,a) \right] & \leqslant
\sum_{t = 1}^{T} \mathbb{E}_{\rho_h^{(t)}} \delta_h^{(t)}(s,a) \left( \frac{\mathbb{E}_{\nu_h} [\delta_h^{(t)}(s,a)]^2 }{\mathbb{E}_{\nu_h} [\delta_h^{(t)}(s,a)]^2} \right)^{\frac{1}{2}} \\
& \leqslant \sqrt{\sum_{t=1}^T  \frac{[\mathbb{E}_{\rho_h^{(t)}} \delta_h^{(t)}(s,a)]^2}{\mathbb{E}_{\nu_h} [\delta_h^{(t)}(s,a)]^2} } \sqrt{\sum_{t=1}^T \mathbb{E}_{\nu_h} [\delta_h^{(t)}(s,a)]^2} \\
& \lesssim \sqrt{\sum_{t=1}^T  \frac{[\mathbb{E}_{\rho_h^{(t)}} \delta_h^{(t)}(s,a)]^2}{\mathbb{E}_{\nu_h} [\delta_h^{(t)}(s,a)]^2} }  \sqrt{ \frac{ \beta T}{n_\mathrm{off}}} \\
& \leqslant \sqrt{T \max_{t \leqslant T} \frac{[\mathbb{E}_{\rho_h^{(t)}} \delta_h^{(t)}(s,a)]^2}{\mathbb{E}_{\nu_h} [\delta_h^{(t)}(s,a)]^2} }
\sqrt{ \frac{ \beta T}{n_\mathrm{off}}} \\
& \leqslant C_{\pi} \sqrt{ \frac{ \beta T^2}{n_\mathrm{off}}}\,,
    \end{split}
\end{equation*}
where the second inequality uses the Cauchy-Schwartz inequality and the last inequality holds by the following result
\begin{equation*}
   \forall \pi \in \Pi,~~ \sqrt{\max_{t \leqslant T} \frac{[\mathbb{E}_{\rho_h^{(t)}} \delta_h^{(t)}(s,a)]^2}{\mathbb{E}_{\nu_h} [\delta_h^{(t)}(s,a)]^2} } \leqslant \max_{f \in \mathcal{F}} \frac{|[\mathbb{E}_{\rho_h^{\pi}} \delta_h(s,a)]|}{\sqrt{\mathbb{E}_{\nu_h} [\delta_h(s,a)]^2}} = C_{\pi} \,,
\end{equation*}
defined by Eq.~\eqref{eq:Cpi}.
In our setting, we take $n_\mathrm{off} = T$ for achieving $\mathcal{O}(\sqrt{T})$ regret.

Finally, by taking $\beta = c \log \left( \frac{\mathscr{N}_{\mathcal{F}}(1/T) TH }{ \delta } \right)$ for some constant $c$ and $\delta \in (0,1)$, the regret bound of Algorithm~\ref{alg:fqi} holds with probability at least $1 - \delta$
\begin{equation*}
\begin{split}
      {\tt Regret} & \leqslant \sum_{t=1}^T \sum_{h=1}^H \mathbb{E}_{(x,a)\sim{}\rho_h^{(t)}} [\delta_h^{(t)}(s,a)] \lesssim \mathcal{O} \left( C_{\pi} H \sqrt{\beta T} \right)\,.
\end{split}
\end{equation*}

{\bf Proof of Case 2:}

Our proof differs from that of \cref{thm:golfcw} in how to estimate the in-sample squared Bellman error under Algorithm~\ref{alg:fqi} without the structural assumption. This is also the technical challenge in this work when compared to \cite{song2022hybrid}.

Recall the definition of $\tau_h(s,a)$ in Eq.~\eqref{eq:tauh}, we have
\begin{itemize}
    \item if $t \leqslant \tau_h(s,a)$, we have $\tilde{\rho}_h^{(t)}(s,a) \leqslant [C_{\tt cw}\mu^{\star}_h(s,a)]^p$.
    \item if $t > \tau_h(s,a)$, we have $\tilde{\rho}_h^{(t)}(s,a) > [C_{\tt cw}\mu^{\star}_h(s,a)]^p$ and $\rho^{(t)}(s,a) < C_{\tt cw}^{\frac{1}{p}} [\mu_h^{\star}(s,a)]^{\frac{p-1}{p}}$ in Eq.~\eqref{eq:ceffstar}.
\end{itemize}
Based on this, when $t > \tau_h(s,a)$, the unnormalized measure $\tilde{\rho}_h^{(t)}$ can be upper bounded by 
\begin{equation*}
\begin{split}
      \tilde{\rho}_h^{(t)} & = \tilde{\rho}_h^{(t)} \mathbbm{1}_{\{ {t \leqslant \tau_h(s,a)} \}} + \tilde{\rho}_h^{(t)} \mathbbm{1}_{\{ {t > \tau_h(s,a)} \}}
       \leqslant [C_{\tt cw} \mu_h^{\star}(s,a)]^p + \sum_{i= \tau_h(s,a) + 1}^{t-1} {\rho}_h^{(i)} \\
       & \leqslant [C_{\tt cw} \mu_h^{\star}(s,a)]^p + \sum_{i= \tau_h(s,a) + 1}^{t-1} C_{\tt cw}^{\frac{1}{p}} [\mu_h^{\star}(s,a)]^{\frac{p-1}{p}} \,.
\end{split}
\end{equation*}

Following Eq.~\eqref{eq:reg_CS}, the result on ${\tt I_A}$ can be directly obtained by Eq.~\eqref{eq:term1} in the proof of \cref{thm:golfcw} such that ${\tt I_A}  \lesssim C_{\tt cw}^{\frac{1}{p}} \log T $, and our main effort here is to estimate the in-sample squared Bellman error related to ${\tt I_B}$.
We split it into two terms
   \begin{equation}\label{eq:sbr12}
   \begin{split}
      &    \sqrt{\sum_{t = 1}^{T} \sum_{(s,a)} \tilde{\rho}_h^{(t)}(s,a) \left(\delta_h^{(t)}(s,a)\right)^2 \mathbbm{1}[t \geqslant \tau_h(s,a)]} \\
      & \leqslant \underbrace{\sqrt{\sum_{t = 1}^{T} \sum_{(s,a)} [C_{\tt cw} \mu_h^{\star}(s,a)]^p \left(\delta_h^{(t)}(s,a)\right)^2 }}_{:= {\tt I_{B1}}} + \underbrace{\sqrt{\sum_{t = 1}^{T} \sum_{(s,a)} \sum_{i= \tau_h(s,a) + 1}^t C_{\tt cw}^{\frac{1}{p}} [\mu_h^{\star}(s,a)]^{\frac{p-1}{p}} \left(\delta_h^{(i)}(s,a)\right)^2}}_{:= {\tt I_{B2}}}\,,
   \end{split} 
   \end{equation}
   where we use $\sqrt{a+b} \leqslant \sqrt{a} + \sqrt{b}$ for $a,b \geqslant 0$.

For the first term ${\tt I_{B1}}$ in Eq.~\eqref{eq:sbr12}, using Lemma~\ref{lem:fqioff}, for any $\delta \in (0,1)$, by taking $\beta = c \log \left( \frac{\mathscr{N}_{\mathcal{F}}(1/T) TH }{ \delta } \right)$ for some constant $c$, with probability at least $1 - \delta$, we have
\begin{equation*}
\begin{split}
    {\tt I_{B1}} & \leqslant C_{\tt cw}^{\frac{p}{2}} \sqrt{\sum_{t = 1}^{T} \sum_{(s,a)} \mu_h^{\star}(s,a) \left(\delta_h^{(t)}(s,a)\right)^2 } = C_{\tt cw}^{\frac{p}{2}} \sqrt{\sum_{t = 1}^{T} \sum_{(s,a)} \nu_h(s,a) \frac{\mu_h^{\star}(s,a)}{\nu_h(s,a)} \left(\delta_h^{(t)}(s,a)\right)^2 } \\
    & \lesssim  C_{\tt cw}^{\frac{p}{2}} \sqrt{ \frac{\widetilde{C} \beta T}{n_\mathrm{off}}} \,, 
\end{split}
\end{equation*}
where we use the coverage condition $ \max_{s,a,h} \frac{ \mu_h^{\star}(s,a) }{ \nu_h(s,a) } \leqslant \widetilde{C} $.
At the end of the proof, we discuss the choice of the offline distribution $\nu$.

Similarly, for the second term $ {\tt I_{B2}}$ in Eq.~\eqref{eq:sbr12}, we have
\begin{equation*}
    \begin{split}
         {\tt I_{B2}} & \leqslant C_{\tt cw}^{\frac{1}{2p}} \sqrt{\sum_{t = 1}^{T} \sum_{(s,a)} \sum_{i= 1}^t [\mu_h^{\star}(s,a)]^{\frac{p-1}{p}} \left(\delta_h^{(i)}(s,a)\right)^2} = C_{\tt cw}^{\frac{1}{2p}} \sqrt{\sum_{t = 1}^{T} t \sum_{(s,a)} [\mu_h^{\star}(s,a)]^{\frac{p-1}{p}} \left(\delta_h^{(i)}(s,a)\right)^2} \\
        & \leqslant C_{\tt cw}^{\frac{1}{2p}} \sqrt{\sum_{t = 1}^{T} t \sum_{(s,a)} \sqrt{\mu_h^{\star}(s,a)} \left(\delta_h^{(i)}(s,a)\right)^2} \,.
    \end{split}
\end{equation*}
Using Lemma~\ref{lem:fqioff} and the coverage condition $ \max_{s,a,h} \frac{ \mu_h^{\star}(s,a) }{ \nu^2_h(s,a) } \leqslant \widetilde{C} $, with the same probability as conducted in  ${\tt I_{B1}}$, we have
\begin{equation*}
    \sum_{(s,a)} \sqrt{\mu_h^{\star}(s,a)} \left(\delta_h^{(t)}(s,a)\right)^2 \lesssim \sqrt{\widetilde{C}} \sum_{(s,a)} \nu_h(s,a) \left(\delta_h^{(t)}(s,a)\right)^2 \lesssim \frac{\beta \sqrt{\widetilde{C}}}{n_{\mathrm{off}}}\,.
\end{equation*}
which implies
\begin{equation*}
   {\tt I_{B2}} \lesssim C_{\tt cw}^{\frac{1}{2p}} T \sqrt{\frac{\beta \widetilde{C}^{\frac{1}{2}}} {n_\mathrm{off}}} \leqslant C_{\tt cw}^{\frac{1}{2p}} T \sqrt{\frac{\beta \widetilde{C}} {n_\mathrm{off}}} \,,
\end{equation*}
where we use $\widetilde{C} \geqslant 1$.

Combining the estimation of $ {\tt I_{B1}}$ and $ {\tt I_{B2}}$ into Eq.~\eqref{eq:sbr12}, the in-sample squared Bellman error related to $ {\tt I_{B}}$ can be estimated with probability at least $1 - \delta$
\begin{equation*}
    \sqrt{\sum_{t = 1}^{T} \sum_{(s,a)} \tilde{\rho}_h^{(t)}(s,a) \left(\delta_h^{(t)}(s,a)\right)^2 \mathbbm{1}[t \geqslant \tau_h(s,a)]} \lesssim C_{\tt cw}^{\frac{p}{2}} \sqrt{\frac{\beta T \widetilde{C}}{n_\mathrm{off}}} + C_{\tt cw}^{\frac{1}{2p}} T \sqrt{\frac{\beta \widetilde{C}}{n_\mathrm{off}}}\,.
\end{equation*}

Accordingly, for any $\delta \in (0,1)$, by taking $\beta = c \log \left( \frac{\mathscr{N}_{\mathcal{F}}(1/T) TH }{ \delta } \right)$ for some constant $c$, the regret bound of Algorithm~\ref{alg:fqi} holds with probability at least $1 - \delta$
\begin{equation*}
\begin{split}
      {\tt Regret} & \leqslant \sum_{t=1}^T \sum_{h=1}^H \mathbb{E}_{(x,a)\sim{}d_h^{(t)}} [\delta_h^{(t)}(s,a)] \lesssim \mathcal{O} \left(HC_{\tt cw}^p + H C_{\tt cw}^{\frac{1}{2p}} \sqrt{\log T} \left[ C_{\tt cw}^{\frac{p}{2}} \sqrt{\frac{\beta T \widetilde{C}}{n_\mathrm{off}}} + C_{\tt cw}^{\frac{1}{2p}} T \sqrt{\frac{\beta \widetilde{C}}{n_\mathrm{off}}} \right] \right) \\
      & \lesssim \mathcal{O} \left( C_{\tt cw}^{\frac{1}{p}} H \sqrt{ \frac{ \beta T^2 \widetilde{C} \log T}{n_\mathrm{off}}} \right)\,.
\end{split}
\end{equation*}
If taking $n_\mathrm{off} := T$, we conclude the proof.

\end{proof}

\section{Proofs for Section~\ref{sec:disen}}
\label{app:partial}

In this section, we mainly focus on the proof of \cref{thm:golfpcov} that provides the sample-efficient guarantees of the GOLF algorithm under our partial/rest coverage condition. %This is one of our main contribution in techniques. 
The key difficulty is how to tackle the issue that some occupancy measures cannot be upper bounded by some (scaling) distribution.
Before our proof, we require the following result on the equivalence for $P_{\tt cov}$. 

\subsection{Proof on the equivalence}
\label{app:partialequ}
Based on our definition, it can be easily found that  
$P_{\tt cov} =P_{\tt cr} := \max_{h\in [H] }  \sum_{(s,a) \in \mathcal{S} \times \mathcal{A}}\sup_{\pi \in \mathcal{M}} \rho_h^\pi(s,a) $. 
The proof can be easily given from \cite{xie2023role}, and we present it here just for self-completeness.

\begin{proof}
    For every step $h$, denote
%\begin{equation*}
 %   \hat{\mu}_h^\star := \argmin_{ \mu_h \subseteq \Delta(\mathcal{S}\times \mathcal{A} )} \sup_{\pi\in \mathcal{M}}\, \left\| \frac{\rho_h^\pi}{\mu_{h}}\right\|^p_{L^p(\mathrm{d}\mu_{h})}\,,
%\end{equation*}
\begin{equation}\label{eq:hatmuh}
    \hat{\mu}_h^\star := \argmin_{ \mu_h \subseteq \Delta(\mathcal{S}\times \mathcal{A} )} \sup_{\pi\in \mathcal{M}}\, \left\| \frac{\rho_h^\pi}{\mu_{h}}\right\|_{\infty}\,,
\end{equation}
we have, one hand
\begin{equation}\label{eq:pcovcr}
\begin{split}
    \sum_{(s,a)} \sup_{\pi \in \mathcal{M}} \rho_h^{\pi} (s,a) & = \sum_{(s,a)} \frac{\max_{\pi \in \mathcal{M}}  \rho_h^{\pi} (s,a) }{\hat{\mu}_h^\star(s,a)} \hat{\mu}_h^\star(s,a) \\
    & \leqslant \sum_{(s,a)} \left( \max_{(s,a)} \frac{\max_{\pi \in \mathcal{M}}  \rho_h^{\pi} (s,a) }{\hat{\mu}_h^\star(s,a)} \right) \hat{\mu}_h^\star(s,a) \\
    & \leqslant \sum_{(s,a)} P_{\tt cov} \hat{\mu}_h^\star(s,a) = P_{\tt cov}\,.
\end{split}
\end{equation}
On the other hand, for any $\pi \in \mathcal{M}$, take $\mu_h \propto \max_{\pi \in \mathcal{M}} \rho_h^{\pi}$, we have
\begin{equation*}
    \frac{\rho_h^{\pi}(s,a)}{\mu_h(s,a)} = \frac{\rho_h^{\pi}(s,a) \sum_{(s',a')} \max_{\pi' \in \mathcal{M}} \rho_{h}^{\pi'}(s',a') }{\max_{\pi'' \in \mathcal{M}} \rho_h^{\pi''}(s,a)} \leqslant \sum_{(s',a')} \max_{\pi' \in \mathcal{M}} \rho_h^{\pi'}(s',a') = P_{\tt cr}\,,
\end{equation*}
which implies $P_{\tt cov} \leqslant P_{\tt cr}$.
Combining with Eq.~\eqref{eq:pcovcr}, we conclude $P_{\tt cov} = P_{\tt cr}$.
\end{proof}

\subsection{Proof of \cref{thm:golfpcov}}
\label{app:partialpcov}
%Similar to the proof of Theorem~\ref{thm:golfcw}, we give the similar proof of sample-efficient online RL under the partial coverage condition.
Here we give the proof of sample-efficient guarantees of the GOLF algorithm under the coverage condition regarding the partial/rest policy class.
The key difficulty is how to tackle the issue that some occupancy measures cannot be upper bounded by $P_{\tt cov} \hat{\mu}^{\star}_h$.
We need to build the connection between $\rho_h^{(t)}$ and $\hat{\mu}_h^{\star}$ and introduce $P_{\tt out}$ to control such distribution shift.
\begin{proof}[Proof of \cref{thm:golfpcov}]
    Similar to \cref{thm:golfcw}, the ``exploration'' phase for each state-action pair $(s,a)\in \mathcal{S} \times \mathcal{A}$ based on our partial coverage $P_{\tt cov}$ is defined as
\begin{equation}\label{eq:tauhpcov}
   \hat{\tau}_h(s,a) = \min\left\{t \mid \tilde{\rho}_h^{(t)}(s,a) \geqslant P_{\tt cov}\hat{\mu}^{\star}_h(s,a) \right\}\,.
\end{equation}

Regarding the ``exploration'' phase, we use that $|\delta_h^{(t)}| \leqslant 1$ to bound
\begin{equation*}
\begin{split}
  \sum_{t = 1}^{T} \mathbb{E}_{(s,a)\sim{}\rho_h^{(t)}}\left[\delta_h^{(t)}(s,a)\mathbbm{1}_{\{ t < \hat{\tau}_h(s,a) \} } \right] & \leqslant
  \sum_{(s,a)}\sum_{t<\hat{\tau}_h(s,a)} \rho_h^{(t)}(s,a)=
  \sum_{(s,a)}\tilde{\rho}_h^{(\hat{\tau}_h(s,a))}(s,a) \\
  & = \sum_{(s,a)} [\tilde{\rho}_h^{(\hat{\tau}_h(s,a) - 1)}(s,a) + \rho_h^{(\hat{\tau}_h(s,a) - 1)}(s,a)] \\
  & \leqslant \sum_{(s,a)} P_{\tt cov} \hat{\mu}^{\star}_h(s,a) + 1\\
  & \leqslant 2 P_{\tt cov} \,,
  \end{split}
\end{equation*}
where we use Eq.~\eqref{eq:tauhpcov} in the second inequality.

In the stable phase, we have $\widetilde{\rho}_h^{(t)}(s,a) \geqslant P_{\tt cov} \hat{\mu}^{\star}_h(s,a)$.
Similar to Eq.~\eqref{eq:reg_CS}, we aim to estimate the following quantity 
\begin{equation}\label{eq:stablepcov}
\begin{split}
     & \sum_{t = 1}^{T} \mathbb{E}_{(s,a)\sim{}\rho_h^{(t)}} \left[ \delta_h^{(t)}(s,a) \mathbbm{1}[t \geqslant \tau_h(s,a)] \right] \\ & \leqslant \sqrt{\underbrace{\sum_{t = 1}^{T} \sum_{(s,a)} \frac{\left( \mathbbm{1}[t \geqslant \tau_h(s,a)] \rho_h^{(t)}(s,a) \right)^2}{\tilde{\rho}_h^{(t)}(s,a)} }_{:= {\tt I_A}}} \cdot \sqrt{\underbrace{\sum_{t = 1}^{T} \sum_{(s,a)} \tilde{\rho}_h^{(t)}(s,a) \left(\delta_h^{(t)}(s,a)\right)^2\mathbbm{1}[t \geqslant \tau_h(s,a)]}_{:= {\tt I_B}}}\,,
\end{split}
\end{equation}
where the inequality holds by the Cauchy-Schwarz inequality and ${\tt I_B} \lesssim \mathcal{O}(\beta T)$ w.h.p by \cref{lem:jin2021} from the result of \cite{jin2021bellman}.
Our main effort in this proof is to bound ${\tt I_A}$.

{\bf Bound ${\tt I_A}$:}
If the current policy $\pi_h^{(t)}$ generating $\rho_h^{(t)}(s,a)$ belongs to $\mathcal{M}_h$, according to Eq.~\eqref{eq:hatmuh}, we have
\begin{equation*}
    P_{\tt cov} = \sup_{\pi \in \mathcal{M}, h \in [H]} \left\| \frac{\rho_h^{\pi}}{\hat{\mu}^{\star}_h} \right\|_{\infty} \geqslant
    \max_{(s,a), h \in [H]} \frac{\rho_h^{(t)}(s,a)}{\hat{\mu}^{\star}_h(s,a)} \,,
\end{equation*}
which implies that for any $(s,a) \in \mathcal{S \times A}$, we have $\rho_h^{(t)}(s,a) \leqslant P_{\tt cov} \hat{\mu}^{\star}_h(s,a)$ if $\pi_h^{(t)} \in \mathcal{M}_h$.
Nevertheless, in online RL, we can not ensure $\pi_h^{(t)} \in \mathcal{M}_h$ such that
$\rho_h^{(t)}(s,a) \leqslant P_{\tt cov} \hat{\mu}^{\star}_h(s,a)$, which leads to the main difficulty: how to bound the first term in Eq.~\eqref{eq:stablepcov} if $\pi_h^{(t)} \notin \mathcal{M}_h$.
Accordingly, we split ${\tt I_A}$ into two cases: $\rho_h^{(t)}(s,a) \leqslant P_{\tt cov} \hat{\mu}^{\star}_h(s,a)$ and $\rho_h^{(t)}(s,a) > P_{\tt cov} \hat{\mu}^{\star}_h(s,a)$ as below

\begin{equation}\label{eq:term1pcov}
\begin{split}
    & {\tt I_A} := \sum_{t = 1}^{T} \sum_{(s,a)} \frac{\left( \mathbbm{1}_{ \{t \geqslant \hat{\tau}_h(s,a) \} } \rho_h^{(t)}(s,a) \right)^2}{\tilde{\rho}_h^{(t)}(s,a)}  \leqslant 2 \sum_{t = 1}^{T} \sum_{(s,a)} \frac{\left( \mathbbm{1}_{ \{t \geqslant \hat{\tau}_h(s,a) \} } \rho_h^{(t)}(s,a) \right)^2}{P_{\tt cov} \hat{\mu}^{\star}_h(s,a) + \tilde{\rho}_h^{(t)}(s,a)} \\
    & = 2 \underbrace{\sum_{t = \hat{\tau}_h}^{T} \sum_{(s,a)} \rho_h^{(t)}(s,a) \frac{ \rho_h^{(t)}(s,a)  \mathbbm{1}_{ \big\{ \rho_h^{(t)}(s,a) \leqslant P_{\tt cov} \hat{\mu}^{\star}_h(s,a) \big\} } }{P_{\tt cov} \hat{\mu}^{\star}_h(s,a) + \tilde{\rho}_h^{(t)}(s,a)} }_{{\tt I_{A1}}} + 2 \underbrace{\sum_{t = \hat{\tau}_h}^{T} \sum_{(s,a)} \rho_h^{(t)}(s,a) \frac{ \rho_h^{(t)}(s,a)  \mathbbm{1}_{\big\{ \rho_h^{(t)}(s,a) > P_{\tt cov} \hat{\mu}^{\star}_h(s,a) \big\} } }{P_{\tt cov} \hat{\mu}^{\star}_h(s,a) + \tilde{\rho}_h^{(t)}(s,a)} }_{{\tt I_{A2}}} \,.
\end{split}
\end{equation}
%For term $(II)$, we know that $\pi_h^{(t)} \notin \mathcal{M}_h$. 
%According to Lemma~\ref{lem:boundedpht}, we have
%\begin{equation}\label{eq:boundrho}
%   \rho_h^{\pi^{\star}}(s,a) \leqslant P_{\tt cov} \hat{\mu}^{\star}_h(s,a) < \rho_h^{(t)}(s,a) \leqslant P_{\tt cov} \hat{\mu}^{\star}_h(s,a) + \rho_h^{\pi^{\star}}(s,a) \leqslant 2 P_{\tt cov} \hat{\mu}^{\star}_h(s,a) \,.
%\end{equation}

{\bf Bound ${\tt I_{A1}}$:} Since $\rho_h^{(t)}(s,a) \leqslant P_{\tt cov} \hat{\mu}^{\star}_h(s,a)$ satisfies the condition in Lemma~\ref{lem:per_sa_ep},
similar to Eq.~\eqref{eq:term1}, term ${\tt I_{A1}}$ can be estimated by
\begin{equation*}
\begin{split}
    {\tt I_{A1}} & \lesssim \sum_{(s,a)} \max_{i \leqslant T} \rho_h^{(i)}(s,a) \sum_{t=1}^T \frac{ \rho_h^{(t)}(s,a)  \mathbbm{1}_{ \big\{ \rho_h^{(t)}(s,a) \leqslant P_{\tt cov} \hat{\mu}^{\star}_h(s,a) \big\} } }{P_{\tt cov} \hat{\mu}^{\star}_h(s,a) + \tilde{\rho}_h^{(t)}(s,a)} \\
    & \lesssim \sum_{(s,a)} P_{\tt cov} \hat{\mu}^{\star}_h(s,a)  \log T \\
    & = P_{\tt cov}  \log T\,.
\end{split} 
\end{equation*}

{\bf Bound ${\tt I_{A2}}$:} We cast the regime $\rho_h^{(t)}(s,a) > P_{\tt cov} \hat{\mu}^{\star}_h(s,a)$ into two cases: 
\begin{itemize}
    \item {\bf Case 1:} $\rho_h^{(t)}(s,a) \leqslant c_1 P_{\tt cov} \hat{\mu}^{\star}_h(s,a)$ for any $(s,a) \in \mathcal{S} \times \mathcal{A}$ and some constant $c_1 \geqslant 1$.
    \item {\bf Case 2:} $\rho_h^{(t)}(s,a) > c_1 P_{\tt cov} \hat{\mu}^{\star}_h(s,a)$ for all potential $(s,a)$.
\end{itemize}
Recall the definition of ${\mathcal{B}^{\bar{\mathcal{M}}}}$ in Definition~\ref{def:ppar}, we consider a special case 
\begin{equation*}
    \mathcal{B}^{(t)}_h := \left\{ (s,a) \in \mathcal{S} \times \mathcal{A} \mid \rho_h^{(t)}(s,a) > c_1 P_{\tt cov} \hat{\mu}^{\star}_h(s,a) \right\}\,, \quad h \in [H]\,.
\end{equation*}

{\bf Case 1:} $\rho_h^{(t)}(s,a) \leqslant c_1 P_{\tt cov} \hat{\mu}^{\star}_h(s,a)$. \\
We split term ${\tt I_{A2}}$ into two parts
\begin{equation*}
    {\tt I_{A2}} = \underbrace{\sum_{t = \hat{\tau}_h}^{T} \sum_{(s,a)} \rho_h^{(t)}(s,a) \frac{ \rho_h^{\bar{\pi}}(s,a)  \mathbbm{1}_{\big\{ \rho_h^{(t)}(s,a) > P_{\tt cov} \hat{\mu}^{\star}_h(s,a) \big\} } }{P_{\tt cov} \hat{\mu}^{\star}_h(s,a) + \tilde{\rho}_h^{(t)}(s,a)} }_{(II_1)} + \underbrace{\sum_{t = \hat{\tau}_h}^{T} \sum_{(s,a)} \rho_h^{(t)}(s,a) \frac{ [\rho_h^{(t)}(s,a) - \rho_h^{\bar{\pi}}(s,a) ]  \mathbbm{1}_{\big\{ \rho_h^{(t)}(s,a) > P_{\tt cov} \hat{\mu}^{\star}_h(s,a) \big\} } }{P_{\tt cov} \hat{\mu}^{\star}_h(s,a) + \tilde{\rho}_h^{(t)}(s,a)} }_{(II_2)} \,.
\end{equation*}
For $(II_1)$, we know $\rho_h^{\bar{\pi}}(s,a) \leqslant P_{\tt cov} \hat{\mu}^{\star}_h(s,a)$ due to $\bar{\pi} \in \mathcal{M}$, we have
\begin{equation}\label{eq:newbound}
    \begin{split}
        \sum_{t=1}^T \frac{ \rho_h^{\bar{\pi}}(s,a)  }{P_{\tt cov} \hat{\mu}^{\star}_h(s,a) + \tilde{\rho}_h^{(t)}(s,a)} & \leqslant 2 \sum_{t=1}^T \log \left( 1 + \frac{ \rho_h^{\bar{\pi}}(s,a)  }{P_{\tt cov} \hat{\mu}^{\star}_h(s,a) + \tilde{\rho}_h^{(t)}(s,a)} \right) \leqslant 2 \sum_{t=1}^T \log \left( 1 + \frac{ \rho_h^{{(t)}}(s,a)  }{P_{\tt cov} \hat{\mu}^{\star}_h(s,a) + \tilde{\rho}_h^{(t)}(s,a)} \right) \\
        & = 2 \log \left( \prod_{t=1}^T \frac{P_{\tt cov} \hat{\mu}^{\star}_h(s,a) + \sum_{i=1}^t \rho_h^{(i)}(s,a) }{ P_{\tt cov} \hat{\mu}^{\star}_h(s,a) + \sum_{i=1}^{t-1} \rho_h^{(i)}(s,a)} \right) = 2 \log \left( 1 +  \frac{ \sum_{i=1}^T \rho_h^{(i)}(s,a) }{ P_{\tt cov} \hat{\mu}^{\star}_h(s,a) } \right) \\
        & \leqslant 2 \log (1+ c_1 T)\,,
    \end{split}
\end{equation}
where in the first inequality we use $x \leqslant 2 \log(1+x)$ for any $x \in [0,1]$; the second inequality holds by $\rho_h^{\bar{\pi}}(s,a) < P_{\tt cov} \hat{\mu}^{\star}_h(s,a) < \rho_h^{(t)}(s,a)$ and the last inequality uses the condition of {\bf Case 1}.
Based on this result, we can upper bound term $(II_1)$ such that
\begin{equation*}
\begin{split}
    (II_1) & \lesssim \sum_{(s,a)} \max_{i \leqslant T} \rho_h^{(i)}(s,a) \sum_{t=1}^T \frac{ \rho_h^{\bar{\pi}}(s,a)  }{P_{\tt cov} \hat{\mu}^{\star}_h(s,a) + \tilde{\rho}_h^{(t)}(s,a)} \\
    & \lesssim \sum_{(s,a)} [c_1 P_{\tt cov} \hat{\mu}^{\star}_h(s,a) ] \log (1 + c_1 T ) \\
    & \lesssim c_1 P_{\tt cov} \log T\,.
\end{split} 
\end{equation*}

For $(II_2)$, since $\rho_h^{(t)}(s,a) - \rho_h^{\bar{\pi}}(s,a) \leqslant c_1 P_{\tt cov} \hat{\mu}^{\star}_h(s,a)$, similar to Eq.~\eqref{eq:newbound}, we have
\begin{equation*}
    \sum_{t = 1}^{T} \frac{ [\rho_h^{(t)}(s,a) - \rho_h^{\bar{\pi}}(s,a) ]   }{P_{\tt cov} \hat{\mu}^{\star}_h(s,a) + \tilde{\rho}_h^{(t)}(s,a)} \leqslant \sum_{t = 1}^{T} \frac{c_1 P_{\tt cov} \hat{\mu}^{\star}_h(s,a)   }{P_{\tt cov} \hat{\mu}^{\star}_h(s,a) + \tilde{\rho}_h^{(t)}(s,a)} \leqslant 2c_1 \log(1+ c_1 T)\,,
\end{equation*}
which implies $(II_2) \lesssim c_1 P_{\tt cov} \log T$.

{\bf Case 2:} $(s,a) \in \mathcal{B}^{(t)}_h$.
Note that if we choose the reference policy $\bar{\pi} := \pi^{\star}$, according to the definition of $\mathcal{B}^{(t)}_h$ under this case, we have
\begin{equation}\label{eq:pilowp}
    1 \geqslant \sum_{(s,a) \in \mathcal{B}^{(t)}_h}\rho_h^{(t)}(s,a) > \sum_{(s,a) \in \mathcal{B}^{(t)}_h} c_1 P_{\tt cov} \hat{\mu}^{\star}_h(s,a) \geqslant \sum_{(s,a) \in \mathcal{B}^{(t)}_h} c_1 \rho_h^{\pi^{\star}}(s,a)\,,
\end{equation}
which implies that the probability that $\pi_h^{\star}$ visits this state-action pair set $\mathcal{B}^{(t)}_h$ is smaller than $1/c_1$.
That means, we can still identify the optimal policy $\pi^{\star}$ with probability at least $(1 - \frac{1}{c_1})^H \geqslant 1 - \frac{H}{c_1}$ for a proper $c_1$ even though we do not consider $\mathcal{B}^{(t)} := \{ \mathcal{B}^{(t)}_h \}_{h=1}^H$.
%Hence in our analysis, we take $c_1:= H/\delta_1$ for some $\delta_1 \in (0,1)$.
%\fh{this only means, w.h.p for identifying $\pi^{\star}$ but does not mean, our algorithm can find avoid it w.h.p.}

For general reference policy $\bar{\pi}$, we have the following result.
According to the definition of $\tilde{\rho}_h^{(t)}$ for any $t > \hat{\tau}_h$, each component at episode $t$ in term $(II)$ admits
\begin{equation*}
 \sum_{(s,a)} \rho_h^{(t+1)}(s,a) \frac{ \rho_h^{(t+1)}(s,a) }{P_{\tt cov} \hat{\mu}^{\star}_h(s,a) + \tilde{\rho}_h^{(t+1)}(s,a)} \leqslant \sum_{(s,a)} \rho_h^{(t+1)}(s,a) \frac{ \rho_h^{(t+1)}(s,a) }{P_{\tt cov} \hat{\mu}^{\star}_h(s,a) + c_1 P_{\tt cov} \hat{\mu}^{\star}_h(s,a) +\tilde{\rho}_h^{(t)}(s,a)}\,,
\end{equation*}
due to $\rho_h^{(t)}(s,a) > c_1 P_{\tt cov} \hat{\mu}^{\star}_h(s,a)$.

Accordingly, for {\bf Case 2}, $\forall (s,a) \in \mathcal{B}_h^{(t)}$, term ${\tt I_{A2}}$ can be estimated by 
\begin{equation}\label{eq:IIlowp1}
    \begin{split}
        {\tt I_{A2}} & \leqslant \sum_{t = \hat{\tau}_h}^{T} \sum_{(s,a) \in \mathcal{B}_h^{(t)}}{\rho_h^{(t)}}\frac{ \rho_h^{(t)}(s,a)  }{P_{\tt cov} \hat{\mu}^{\star}_h(s,a) + \tilde{\rho}_h^{(t)}(s,a)} \leqslant \sum_{t = \hat{\tau}_h}^{T} \sum_{(s,a) \in \mathcal{B}_h^{(t)}}{\rho_h^{(t)}}\frac{ \rho_h^{(t)}(s,a)  }{P_{\tt cov} \hat{\mu}^{\star}_h(s,a) + \tilde{\rho}_h^{(\hat{\tau}_h)}(s,a) + (t - \hat{\tau}_h) P_{\tt cov} \hat{\mu}^{\star}_h(s,a) }  \\
        & \leqslant \sum_{t = \hat{\tau}_h}^{T} \sum_{(s,a) \in \mathcal{B}_h^{(t)}}{\rho_h^{(t)}}\frac{ \rho_h^{(t)}(s,a)  }{  (t - \hat{\tau}_h + 1) P_{\tt cov} \hat{\mu}^{\star}_h(s,a) }\,. 
    \end{split}
\end{equation}
By the Cauchy–Schwarz inequality, we have
\begin{equation}\label{eq:pcovcauchy}
    \begin{split}
        \sum_{(s,a) \in \mathcal{B}_h^{(t)}}{\rho_h^{(t)}}\frac{ \rho_h^{(t)}(s,a)  }{  (t - \hat{\tau}_h + 1) P_{\tt cov} \hat{\mu}^{\star}_h(s,a) } & \leqslant \frac{1}{P_{\tt cov}}\sqrt{\sum_{(s,a) \in \mathcal{B}_h^{(t)}}\frac{ [\rho_h^{(t)}(s,a)]^2  }{  [ \hat{\mu}^{\star}_h(s,a)]^2 }} \cdot \sqrt{\sum_{(s,a) \in \mathcal{B}_h^{(t)}}\frac{ [\rho_h^{(t)}(s,a)]^2  }{  [t - \hat{\tau}_h + 1]^2 }} \\
        & \leqslant \frac{1}{P_{\tt cov} (t - \hat{\tau}_h + 1) }\sqrt{\sum_{(s,a) \in \mathcal{B}_h^{(t)}}\frac{ [\rho_h^{(t)}(s,a)]^2  }{  [ \hat{\mu}^{\star}_h(s,a)]^2 }} \\
        & = \frac{1}{P_{\tt cov} (t - \hat{\tau}_h + 1) } \left\| \frac{\rho_h^{(t)}}{\hat{\mu}^{\star}_h} \mathbbm{1}_{{\mathcal{B}}^{(t)}_h} \right\|_{L^2}\,,
    \end{split}
\end{equation}
where the indicator function $\mathbbm{1}_{{\mathcal{B}}^{(t)}_h} = 1$ if $(s,a) \in {\mathcal{B}}^{(t)}_h$, and otherwise is zero.
Accordingly, taking this equation back to Eq.~\eqref{eq:IIlowp1}, we have
\begin{equation*}
    \begin{split}
        {\tt I_{A2}} & \leqslant \sum_{t=\hat{\tau}_h}^T \frac{1}{P_{\tt cov} (t - \hat{\tau}_h + 1) } \left\| \frac{\rho_h^{(t)}}{\hat{\mu}^{\star}_h} \mathbbm{1}_{{\mathcal{B}}^{(t)}_h} \right\|_{L^2} \lesssim \frac{\log T}{P_{\tt cov}} \max_{t \leqslant T} \left\| \frac{\rho_h^{(t)}}{\hat{\mu}^{\star}_h} \mathbbm{1}_{{\mathcal{B}}^{(t)}_h} \right\|_{L^2} \,.
    \end{split}
\end{equation*}

Accordingly, combining the results of ${\tt I_{A1}}$ and ${\tt I_{A2}}$ into Eq.~\eqref{eq:term1pcov}, under the definition of $\mathcal{B}^{(t)}_h$, we have
\begin{equation*}
   {\tt I_{A}} = \sum_{t = 1}^{T} \sum_{(s,a)} \frac{\left( \mathbbm{1}_{ \{t \geqslant \hat{\tau}_h(s,a) \} } \rho_h^{(t)}(s,a) \right)^2}{\tilde{\rho}_h^{(t)}(s,a)}  \lesssim  \left( c_1 P_{\tt cov} + \frac{1}{P_{\tt cov}} \max_{\hat{\tau}_h \leqslant t \leqslant T} \left\| \frac{\rho_h^{(t)}}{\hat{\mu}^{\star}_h} \mathbbm{1}_{{\mathcal{B}}^{(t)}_h} \right\|_{L^2} \right) \log T\,.
\end{equation*}

Following Eq.~\eqref{eq:stablepcov}, by taking $\beta = c \log \left( \frac{\mathscr{N}_{\mathcal{F}}(1/T) TH }{ \delta } \right)$ for some constant $c$ and $\delta \in (0,1)$, combining the results of ${\tt I_{A}}$ and ${\tt I_{A}}$, the result for the stable phase holds with probability at least $1 - \delta$
\begin{equation*}
\begin{split}
&~ \sum_{t = 1}^{T} \mathbb{E}_{(s,a)\sim{}\rho_h^{(t)}}\left[\delta_h^{(t)}(s,a)\mathbbm{1}[t \geqslant \hat{\tau}_h(s,a)]\right]
\\
& \leqslant \sqrt{\sum_{t = 1}^{T} \sum_{(s,a)} \frac{\left( \mathbbm{1}[t \geqslant \hat{\tau}_h(s,a)] \rho_h^{(t)}(s,a) \right)^2}{\tilde{\rho}_h^{(t)}(s,a)} } \cdot  \sqrt{\sum_{t = 1}^{T} \sum_{(s,a)} \tilde{\rho}_h^{(t)}(s,a) \left(\delta_h^{(t)}(s,a)\right)^2\mathbbm{1}[t \geqslant \hat{\tau}_h(s,a)]} \\
& \lesssim \mathcal{O} \left(\sqrt{ c_1 P_{\tt cov} + \frac{1}{P_{\tt cov}} \max_{\hat{\tau}_h \leqslant t \leqslant T} \left\| \frac{\rho_h^{(t)}}{\hat{\mu}^{\star}_h} \mathbbm{1}_{{\mathcal{B}}^{(t)}_h} \right\|_{L^2} } \sqrt{\beta T \log T} \right) \\
& \lesssim \mathcal{O} \left( \left( \sqrt{ c_1 P_{\tt cov} }  + \frac{1}{\sqrt{P_{\tt cov}}} \max_{\hat{\tau}_h \leqslant t \leqslant T} \left\| \frac{\rho_h^{(t)}}{\hat{\mu}^{\star}_h} \mathbbm{1}_{{\mathcal{B}}^{(t)}_h} \right\|_{L^2}^{\frac{1}{2}} \right)  \sqrt{\beta T \log T} \right)\,,
\end{split}
\end{equation*}
where the last inequality uses $\sqrt{a+b} \leqslant \sqrt{a} + \sqrt{b}$ for any $a,b \geqslant 0$.
Besides, we can set the quantity to $\max \left\{1, \max_{\hat{\tau}_h \leqslant t \leqslant T} \left\| \frac{\rho_h^{(t)}}{\hat{\mu}^{\star}_h} \mathbbm{1}_{{\mathcal{B}}^{(t)}_h} \right\|_{L^2} \right\}$ such that the square root operator can be taken into the $\max$.
If this quantity is smaller than 1, that means $ \sqrt{ c_1 P_{\tt cov} } $ dominates the result and thus the second can be omitted.
Recall the definition of $P_{\tt out}$ in Definition~\ref{def:ppar}, we have 
\begin{equation*}
    \max_{h \in [H],\pi \notin \mathcal{M}} \left\| \frac{\rho_h^{(t)}}{\hat{\mu}^{\star}_h} \mathbbm{1}_{{\mathcal{B}}^{(t)}_h} \right\|_{L^2}^{\frac{1}{2}} \leqslant P_{\tt out} \,.
\end{equation*}
Accordingly, our regret bound holds with probability at least $1 - \delta$
\begin{equation}\label{eq:regretthm2}
\begin{split}
    {\tt Regret} & \leqslant \sum_{t=1}^T \sum_{h=1}^H \mathbb{E}_{(x,a)\sim{}\rho_h^{(t)}} [\delta_h^{(t)}(s,a)] \lesssim \mathcal{O} \left(HP_{\tt cov} + H\Big( \sqrt{c_1 P_{\tt cov} }  + \frac{ P_{\tt out} }{\sqrt{ P_{\tt cov} } } \Big)  \sqrt{\beta T \log T} \right) \\
    & = \mathcal{O} \left( H\Big( \sqrt{c_1 P_{\tt cov} }  + \frac{ P_{\tt out} }{\sqrt{ P_{\tt cov} } } \Big)  \sqrt{\beta T \log T} \right) \,.
\end{split}
\end{equation}
If $\mathcal{B}_h$ is an empty set for some $h$, it means that $\rho_h^{(t)}(s,a) < c_1 P_{\tt cov} \hat{\mu}^{\star}_h(s,a)$ always holds for any $(s,a) \in \mathcal{S} \times \mathcal{A}$, which falls into the $\mathcal{M} = \Pi$ case.
In this case, we have $P_{\tt cov} = C_{\tt cov}$, and the second term with $P_{\tt out} = 0$ in the above equation is discarded.
Hence we can recover the result of \cite{xie2023role}.

Clearly, there exists a trade-off between $P_{\tt cov}(\zeta)$ and $P_{\tt out}(\zeta)$ that depends on $\zeta$.
That means, there exists a proper $\zeta^{\star}$ such that  $P_{\tt out}(\zeta) = \sqrt{c_1} P_{\tt cov}(\zeta^{\star}) $ by the property of the function $x + {c}/{x}$ for some constant $c$. Accordingly, the regret bound in Eq.~\eqref{eq:regretthm2} can be improved to
\begin{equation*}
\begin{split}
    {\tt Regret} & \lesssim  \mathcal{O} \left( H \sqrt{c^{\frac{1}{2}}_1 \beta T P_{\tt out}(\zeta^{\star}) \log T } \right)
\end{split}
\end{equation*}
which admits $P_{\tt out}(\zeta^{\star}) \leq C_{\tt cov}$. 
This demonstrates a better regret bound than \cite{xie2023role} by a good trade-off between $P_{\tt cov}$ and $P_{\tt out}$.
Finally, we conclude the proof.
\end{proof}

\section{Proof for Section~\ref{sec:linearmdp}}

In this section, we first prove \cref{thm:lsvi} in \cref{app:lsvi} and then discuss the choice of the regularization parameter in \cref{app:regu}.

\subsection{Proof of \cref{thm:lsvi}}
\label{app:lsvi}

To prove our result, we need the following notations and lemmas to aid our proof.
For self-completeness, we include the LSVI-UCB algorithm \cite{jin2020provably} for linear MDP, see Algorithm~\ref{algo:lsviucb} for details.

\begin{algorithm}[thb]
		\caption{LSVI-UCB for linear MDP \cite{jin2020provably}}\label{algo:lsviucb}
		\begin{algorithmic}[1]
  \STATE{\textbf{Input:} The regularization parameter $\lambda$ and confidence parameter $\beta$.}
			\FOR{episode $t = 1, \ldots, T$}
			\STATE{Receive the initial state $s^t_1$ and set $V_{H+1}^t$ as the zero function.}
			\FOR{step $h = H, \ldots, 1$}
			\STATE Obtain $\Lambda^t_h \leftarrow \sum_{\tau =1}^t [\phi(s_h^{\tau}, a_h^{\tau}) \phi(s_h^{\tau}, a_h^{\tau})^{\!\top} ] + \lambda I$ 
			\STATE Obtain $\widehat{\bm w}_h^t \leftarrow (\Lambda^t_h)^{-1} \sum_{\tau =1}^t \phi(s_h^{\tau}, a_h^{\tau}) [r_h(s_h^{\tau}, a_h^{\tau}) + \max_{a \in \mathcal{A}} Q_{h+1}^t(s_{h+1}^{\tau},a)]$ and  $\widehat{Q}^t_h(\cdot,\cdot) = \langle \phi(\cdot,\cdot), \widehat{\bm w}_h^{t} \rangle$
			\STATE 
   Obtain $Q_h^t(\cdot,\cdot) \leftarrow \min\{ \widehat{Q}^t_h(\cdot,\cdot) + \beta [\phi(\cdot,\cdot)^{\!\top} (\Lambda^t_h)^{-1} \phi(\cdot,\cdot) ]^{1/2}, H  \}$
			\ENDFOR
   \FOR{step $h = 1, \ldots, H$}
			\STATE Take action $a^t_h \gets  \argmax_{a \in \mathcal{A} } Q_h^t (s^t_h, a)$ and obtain $V_h^t (\cdot) = \max_{a\in \mathcal{A}} Q_h^t(\cdot, a)$. 
			\STATE Observe the reward $r_h(s_h^t, a_h^t)$ and the next state $s^t_{h+1}$. 
			\ENDFOR
			\ENDFOR
		\end{algorithmic}
	\end{algorithm}	

In LSVI-UCB, the estimator is given by solving a regularized least squares problem as below.
 \begin{equation}\label{eq:lsvi}
 \widehat{\bm w}_h^t \leftarrow \argmin_{\bm w \in \mathbb{R}^d} \sum_{\tau =1}^{t-1} [r_h(s_h^{\tau}, a_h^{\tau}) + \max_{a \in \mathcal{A}} Q_{h+1}^t(s_{h+1}^{\tau},a) - \langle \bm w, \phi(s_h^{\tau}, a_h^{\tau}) \rangle ]^2 + \lambda \| \bm w \|_2^2 \,,
 \end{equation}
where the feature mapping $\phi(s,a) \in \mathbb{R}^d$ satisfies $\| \phi(s,a)\|_2 \leqslant 1$ and $\lambda \geqslant 1$ is the regularization parameter.
For notational simplicity, denote 
    \begin{equation}\label{eq:lambdaht}
    \Lambda_h^t := \lambda I + \sum_{i=1}^{t-1} \phi(s_h^i,a_h^i) \phi(s_h^i,a_h^i)^{\!\top} := \lambda I + (\Phi_h^t)^{\!\top} \Phi_h^t\,, \quad \mbox{with}~(s^i_h,a^i_h) \sim \rho_h^{(t)}\,,
\end{equation}
where $\Phi_h^t = [\phi_h(s_h^1,a_h^1), \cdots, \phi_h(s_h^{t-1},a_h^{t-1})]^{\!\top} \in \mathbb{R}^{(t-1) \times d}$, and accordingly we can easily obtain an estimation of eigenvalues of $(\Lambda_h^t)^{-1}$ such that
\begin{equation}\label{eq:lambdamin}
   \frac{1}{\lambda} \geqslant \lambda_{\max} [(\Lambda_h^t)^{-1}] \geqslant \lambda_{\min} [(\Lambda_h^t)^{-1}] = \frac{1}{\lambda_{\max} [ (\Phi_h^t)^{\!\top} \Phi_h^t + \lambda I] } \geqslant \frac{1}{\lambda_{\max} [ (\Phi_h^t)^{\!\top} \Phi_h^t] + \lambda } \geqslant \frac{1}{d+\lambda}\,,
\end{equation}
where the last inequality holds by $\| \phi(s,a) \|_2 \leqslant 1$ and the fact 
$\| \bm A \|_2 \leqslant \sqrt{mn} \max_{i,j} A_{ij}$ where $\bm A \in \mathbb{R}^{m \times n}$. 

In the next, we have the following lemmas.

\begin{lemma}\label{lem:changemea}
For the intermediate quantity $\phi(s^i_h,a^i_h)^{\!\top} (\Lambda_h^t)^{-1} \phi(s^i_h,a^i_h)$ with $i \in [T]$, where $\Lambda_h^t$ defined by Eq.~\eqref{eq:lambdaht} realized by the occupancy measure $\rho_h^{(t)}$, and the feature mapping $\phi(s^i_h,a^i_h)$ is assumed to admit $ (s^i_h,a^i_h) \overset{\text{i.i.d}}{\sim} \mu_h$ for a underlying distribution $\mu_h$, then we have
\begin{equation*}
     \frac{1}{T} \sum_{i=1}^T [  \phi(s^i_h,a^i_h)^{\!\top} (\Lambda_h^t)^{-1} \phi(s^i_h,a^i_h) ] \leqslant \frac{2d^2}{T \lambda} \log (T+1) \,.
\end{equation*}
\end{lemma}

\begin{proof}
   We introduce an auxiliary variable $\widetilde{\Lambda}_h^t \in \mathbb{R}^{d \times d}$ such that
\begin{equation*}
    \widetilde{\Lambda}_h^t = \lambda I + \sum_{j=1}^{t-1} \phi(s_h^j,a_h^j) \phi(s_h^j,a_h^j)^{\!\top} \quad \mbox{with}~(s^j_h,a^j_h) \overset{\text{i.i.d}}{\sim} \mu_h\,,
\end{equation*}
then we have
\begin{equation*}
    \begin{split}
        \frac{1}{T} \sum_{i=1}^T [  \phi(s^i_h,a^i_h)^{\!\top} (\Lambda_h^t)^{-1} \phi(s^i_h,a^i_h) ] & =  \frac{1}{T} \sum_{i=1}^T \left(  \phi(s^i_h,a^i_h)^{\!\top} (\widetilde{\Lambda}_h^{(i-1)})^{-1} [\widetilde{\Lambda}_h^{(i-1)} (\Lambda_h^t)^{-1} ] \phi(s^i_h,a^i_h) \right) \\
        & = \frac{1}{T} \sum_{i=1}^T \mathrm{Tr} \left( \phi(s^i_h,a^i_h) \phi(s^i_h,a^i_h)^{\!\top} (\widetilde{\Lambda}_h^{(i-1)})^{-1} [\widetilde{\Lambda}_h^i (\Lambda_h^t)^{-1} ]  \right) \\
        & \overset{(a)}{\leqslant} \frac{1}{T} \sum_{i=1}^T \mathrm{Tr} \left( \phi(s^i_h,a^i_h) \phi(s^i_h,a^i_h)^{\!\top} (\widetilde{\Lambda}_h^{(i-1)})^{-1}   \right) \| \widetilde{\Lambda}_h^{(i-1)} (\Lambda_h^t)^{-1} \|_2 \\
        & \overset{(b)}{\leqslant} \frac{d}{T \lambda} \sum_{i=1}^T \left(  \phi(s^i_h,a^i_h)^{\!\top} (\widetilde{\Lambda}_h^{(i-1)})^{-1} \phi(s^i_h,a^i_h) \right) \\
        & \overset{(c)}{\leqslant} \frac{2d^2}{T \lambda} \log (T+1)\,,
    \end{split}
\end{equation*}
where $(a)$ uses $\mathrm{Tr}(\bm A \bm B) \leqslant \mathrm{Tr}(\bm A) \| \bm B \|_2$; $(b)$ uses $\| (\Lambda_h^t)^{-1} \|_2 \leqslant \frac{1}{\lambda}$, $\| \widetilde{\Lambda}_h^i \|_2 \leqslant d $ via $\| \bm A \|_2 \leqslant \sqrt{mn} \max_{i,j} A_{ij}$ where $\bm A \in \mathbb{R}^{m \times n}$; and (c) uses the elliptical potential lemma with $U_t = U_{t-1} + X_t X_t^{\!\top} \in \mathbb{R}^{d \times d}$, $U_0 = \lambda I$, and $\| X_t \|_2 \leqslant 1$ such that 
\begin{equation*}
    \sum_{t=1}^T X_t^{\!\top} U_{t-1} X_t \leqslant 2 d \log \left(1+ \frac{T}{\lambda d} \right)\,.
\end{equation*}
\end{proof}

\begin{lemma}\label{lem:rhomu}
    Under Assumption~\ref{assum:uef} with $\gamma > 0$ and the feature mapping $\phi(s,a) \in \mathbb{R}^d$ in linear MDP satisfies $\| \phi(s,a)\|_2 \leqslant 1$, we have 
    \begin{equation*}
        \mathbb{E}_{\rho_h^{(t)}} \big[ \phi(s_h,a_h)^{\!\top} (\Lambda_h^t)^{-1} \phi(s_h,a_h) \big]  \leqslant \frac{(d+\lambda)^2}{d^2 \gamma^2 \lambda} \left( \mathbb{E}_{\mu_h} [  \phi(s_h,a_h)^{\!\top} (\Lambda_h^t)^{-1} \phi(s_h,a_h) ] \right)^2 \,.
    \end{equation*}
\end{lemma}

\begin{proof}

Assumption~\ref{assum:uef} yields
$\mathbb{E}_{\mu} [\| \phi(s,a) \|_2^2] \geqslant d \gamma$, by taking $C_e := \frac{1}{d^2 \gamma^2}$, we have
\begin{equation}\label{eq:dgamma}
    \mathbb{E}_{\rho_h^{(t)}} [ \| \phi(s,a) \|_2^2 ]  \leqslant 1 \leqslant C_e \left( \mathbb{E}_{\mu} [\| \phi(s,a) \|_2^2] \right)^2\,.
\end{equation}

Using the linearity of the trace operator and expectation, we have
\begin{equation*}
\begin{split}
       \mathbb{E}_{\rho_h^{(t)}} \big[ \phi(s_h,a_h)^{\!\top} (\Lambda_h^t)^{-1} \phi(s_h,a_h) \big] & = \mathrm{Tr} \left( \mathbb{E}_{\rho_h^{(t)}} \big[ \phi(s_h,a_h) \phi(s_h,a_h)^{\!\top} (\Lambda_h^t)^{-1} \big] \right) \\
       & \leqslant \frac{1}{\lambda} \mathrm{Tr} \left( \mathbb{E}_{\rho_h^{(t)}} \big[ \phi(s_h,a_h) \phi(s_h,a_h)^{\!\top} \big] \right) = \frac{1}{\lambda} \mathbb{E}_{\rho_h^{(t)}} [ \| \phi(s,a) \|_2^2 ] \,,
\end{split}
\end{equation*}
where we use $\| (\Lambda_h^t)^{-1} \|_2 \leqslant 1/\lambda $.
Accordingly, we have
\begin{equation*}
    \begin{split}
        \mathbb{E}_{\rho_h^{(t)}} \big[ \phi(s_h,a_h)^{\!\top} (\Lambda_h^t)^{-1} \phi(s_h,a_h) \big]
        & \leqslant \frac{C_e}{\lambda}  \left( \mathbb{E}_{\mu_h} [ \| \phi(s_h,a_h) \|_2^2 ] \right)^2 \quad \mbox{[using Eq.~\eqref{eq:dgamma}]} \\
        & = \frac{(d+\lambda)^2 C_e}{\lambda} \left( \frac{1}{d+\lambda} \mathbb{E}_{\mu_h} [ \| \phi(s_h,a_h) \|_2^2 ] \right)^2 \\
        & \overset{(a)}{\leqslant}  \frac{(d+\lambda)^2 C_e}{\lambda}  \left( \mathbb{E}_{\mu_h} [ \lambda_{\min}[(\Lambda_h^t)^{-1}] \| \phi(s_h,a_h) \|_2^2 ] \right)^2 \\
        & \overset{(b)}{\leqslant}  \frac{(d+\lambda)^2 C_e}{\lambda}  \left( \mathbb{E}_{\mu_h} \mathrm{Tr} [  \phi(s_h,a_h) \phi(s_h,a_h)^{\!\top} (\Lambda_h^t)^{-1} ] \right)^2 \\
        & = \frac{(d+\lambda)^2}{d^2 \gamma^2 \lambda}  \left( \mathbb{E}_{\mu_h} [  \phi(s_h,a_h)^{\!\top} (\Lambda_h^t)^{-1} \phi(s_h,a_h) ] \right)^2\,,
    \end{split}
\end{equation*}
where $(a)$ uses Eq.~\eqref{eq:lambdamin} and
$(b)$ uses the fact that $\mathrm{Tr}(\bm A \bm B) \geqslant \lambda_{\min} (\bm A) \mathrm{Tr}(\bm B)$ for two PSD matrices $\bm A$ and $\bm B$.
\end{proof}

\begin{lemma}[regret decomposition]
\label{lem:regretdec}
Consider linear MDP with the feature mapping $\phi(s,a) \in \mathbb{R}^d$ satisfying $\| \phi(s,a)\|_2 \leqslant 1$, under Assumption~\ref{assum:uef} with $\gamma > 0$, using LSVI-UCB with the regularization parameter $\lambda$ and a bonus parameter $\beta := \widetilde{\mathcal{O}} \left( \sqrt{\lambda}H(d+ \sqrt{\log \frac{1}{\delta}}) \right)$ with $0 < \delta < 1$, then with probability at least $1 - \delta$, for a underlying distribution $\mu$, the regret admits
\begin{equation*}
    {\tt Regret}(T) \leqslant \frac{2\beta (d+\lambda)}{d \gamma \sqrt{\lambda}} \sum_{h=1}^H \sum_{t=1}^T \mathbb{E}_{\mu_h} [  \phi(s_h,a_h)^{\!\top} (\Lambda_h^t)^{-1} \phi(s_h,a_h) ] \,,
\end{equation*}
where $(s_h,a_h)$ is iid sampled from $\mu_h$.
\end{lemma}

\begin{proof}
Recall the definition of $\beta$ in LSVI-UCB \cite{jin2020provably} with $0 < \delta < 1$
\begin{equation*}
    \beta := \widetilde{\mathcal{O}} \left( \sqrt{\lambda}H \left(d+ \sqrt{\log \frac{1}{\delta}} \right) \right)\,,
\end{equation*}
then according to \cite{jiang2022notes}, with probability at least $1 - \delta$, we have the following regret decomposition
    \begin{equation*}
        {\tt Regret}(T)  \leqslant \sum_{h=1}^H \sum_{t=1}^T \mathbb{E}_{\rho_h^{(t)}} \big[2 \beta \| \phi(s_h,a_h) \|_{(\Lambda_h^t)^{-1}} \big] \,,
    \end{equation*}
where $(s_h,a_h)$ is sampled from the occupancy measure $\rho_h^{(t)}$.
In the next, we conduct the change-of-measure from $\rho_h^{(t)}$ to $\mu_h$, i.e.
\begin{equation*}
    \begin{split}
        {\tt Regret}(T) & \leqslant \sum_{h=1}^H \sum_{t=1}^T \mathbb{E}_{\rho_h^{(t)}} \big[2 \beta \| \phi(s_h,a_h) \|_{(\Lambda_h^t)^{-1}} \big] \\
        & = \sum_{h=1}^H \sum_{t=1}^T 2 \beta \mathbb{E}_{\rho_h^{(t)}} \sqrt{ \phi(s_h,a_h)^{\!\top} (\Lambda_h^t)^{-1} \phi(s_h,a_h) } \\
        & \overset{(a)}{\leqslant} \sum_{h=1}^H \sum_{t=1}^T 2 \beta  \sqrt{ \mathbb{E}_{\rho_h^{(t)}} \big[ \phi(s_h,a_h)^{\!\top} (\Lambda_h^t)^{-1} \phi(s_h,a_h) \big] } \\
        & = \sum_{h=1}^H \sum_{t=1}^T 2 \beta  \sqrt{ \mathrm{Tr} \left( \mathbb{E}_{\rho_h^{(t)}} \big[ \phi(s_h,a_h) \phi(s_h,a_h)^{\!\top} (\Lambda_h^t)^{-1} \big] \right) } \\
        & \overset{(b)}{\leqslant} \frac{2\beta (d+\lambda)}{d \gamma \sqrt{\lambda}} \sum_{h=1}^H \sum_{t=1}^T \mathbb{E}_{\mu_h} [  \phi(s_h,a_h)^{\!\top} (\Lambda_h^t)^{-1} \phi(s_h,a_h) ] \,,
    \end{split}
\end{equation*}
where $(a)$ uses Jensen inequality for the square-root function (concave); $(b)$ uses Lemma~\ref{lem:rhomu}.
\end{proof}

Now we are ready to  prove \cref{thm:lsvi}. 
\begin{proof}

Considering the iid sampling $(s^i_h,a^i_h) \sim \mu_h$ and $0 \leqslant \phi(s^i_h,a^i_h)^{\!\top} (\Lambda_h^t)^{-1} \phi(s^i_h,a^i_h) \leqslant \frac{1}{\lambda}$, denote $\sigma^2 := \mathbb{V}[\phi(s^i_h,a^i_h)^{\!\top} (\Lambda_h^t)^{-1} \phi(s^i_h,a^i_h)] \leqslant \frac{1}{4 \lambda^2}$, then by Bernstein inequality \cite{wainwright2019high}, we have
\begin{equation*}
    \mathrm{Pr} \left[ \left| \frac{1}{T} \sum_{i=1}^T [  \phi(s^i_h,a^i_h)^{\!\top} (\Lambda_h^t)^{-1} \phi(s^i_h,a^i_h) ] - \mathbb{E}_{\mu_h} [  \phi(s_h,a_h)^{\!\top} (\Lambda_h^t)^{-1} \phi(s_h,a_h) ] \right| \geqslant \epsilon \right] \leqslant 2 \exp \left( - \frac{T \epsilon^2}{2(\sigma^2 + \epsilon/\lambda)} \right) \,.
\end{equation*}
That means, with probability at least $1-\delta_1$, we have
\begin{equation}\label{eq:hoeffdiff}
    \mathbb{E}_{\mu_h} [  \phi(s_h,a_h)^{\!\top} (\Lambda_h^t)^{-1} \phi(s_h,a_h) ] \leqslant \frac{1}{T} \sum_{i=1}^T [  \phi(s^i_h,a^i_h)^{\!\top} (\Lambda_h^t)^{-1} \phi(s^i_h,a^i_h) ] +  4 \sqrt{\frac{\sigma^2 \log (2/ \delta_1)}{T}} + \frac{4 \log(2/\delta_1)}{T \lambda}  \,.
\end{equation}

Combining Eq.~\eqref{eq:hoeffdiff} and \cref{lem:changemea} into \cref{lem:regretdec}, for any $\delta \in (0,1)$ and taking $\delta_1 := \delta/2$ and $\beta := \widetilde{\mathcal{O}}( \sqrt{\lambda} d H \log (2/\delta) )$, with probability at least $1 - \delta $, we have 
\begin{equation}\label{eq:regretlambda}
    \begin{split}
        {\tt Regret}(T) & \lesssim \frac{\beta (d+\lambda)}{d \gamma \sqrt{\lambda}}\sum_{h=1}^H \sum_{t=1}^T \left( \frac{d^2}{T\lambda} \log(T+1) +  4 \sqrt{\frac{\sigma^2 \log (4/ \delta)}{T}} + \frac{4 \log(4/\delta)}{T \lambda} \right) \\
        & \lesssim \left( \frac{d+ \lambda}{\gamma \lambda} d^2H^2  \log T   + \frac{d+ \lambda}{\gamma} H^2 \sum_{t=1}^T \sqrt{\frac{\sigma^2}{T}} \right) \log \left( \frac{4}{ \delta} \right)\,.
    \end{split}
\end{equation}
Using $\sigma^2 \lesssim \frac{1}{\lambda^{2\alpha}}$ with $\alpha > 1$ in Assumption~\ref{assum:lowv} and taking $\lambda := T^{\eta}$ with $\eta \in (0,1]$ back to the above regret bound, with probability at least $1 - \delta$, we have
\begin{equation*}
\begin{split}
    {\tt Regret}(T) & \lesssim  \left( \frac{H^2 d^2}{\gamma} \log T + \frac{H^2 \lambda \sigma}{\gamma} \sqrt{T} + \frac{H^2 d}{\gamma} \sigma \sqrt{T} \right) \log \left( \frac{4}{ \delta} \right) \\
    & \lesssim \mathcal{O} \left( \frac{d H^2 }{\gamma} \left(d \log T +  T^{\frac{1}{2}- \eta (\alpha - 1)} \right) \right) \\
    & = \left\{ \begin{array}{rcl}
				\begin{split}
					& \!\!  \mathcal{O} \left(\frac{d^2 H^2}{\gamma} \log T \right) ,~\mbox{if $\eta (\alpha - 1) \geqslant 1/2$} \\
					& \!\! \mathcal{O} \left( \frac{d H^2}{\gamma}  T^{\frac{1}{2}- \eta (\alpha - 1)} \right) ,~\mbox{if $\eta (\alpha \!-\! 1) \in (0,\frac{1}{2})$}  \\
				\end{split}
			\end{array} \right. 
\end{split}
\end{equation*}
which concludes the proof.
\end{proof}

\subsection{Discussion on the regularization parameter}
\label{app:regu}
Recall the regularized least squares in Eq.~\eqref{eq:lsvi}, it is equivalent to
 \begin{equation*}
 \widehat{\bm w}_h^t \leftarrow \argmin_{\bm w \in \mathbb{R}^d} \frac{1}{t-1} \sum_{\tau =1}^{t-1} [r_h(s_h^{\tau}, a_h^{\tau}) + \max_{a \in \mathcal{A}} Q_{h+1}^t(s_{h+1}^{\tau},a) - \langle \bm w, \phi(s_h^{\tau}, a_h^{\tau}) \rangle ]^2 + \lambda' \| \bm w \|_2^2 \,,
 \end{equation*}
 where $\lambda' = \frac{\lambda}{t-1}$.
 The first term is the empirical risk minimization and the second term is the regularizer as Tikhonov regularization.
 The regularization parameter $\lambda' \equiv \lambda'(t) > 0$ admits $\lim_{t \rightarrow \infty} \lambda'(t) = 0$.
 In learning theory, one typically assumes that $\lambda' = \mathcal{O}(t^{-\tau})$ with $\tau \in (0,1]$, decaying with the number of samples \cite{cucker2007learning}, which implies $\lambda = \mathcal{O}(t^{1-\tau})$ in Eq.~\eqref{eq:lsvi}.
 This verifies that our assumption on the regularization parameter makes sense.
  In LSVI-UCB \cite{jin2020provably}, the regularization parameter is chosen as $\lambda = 1$, which implies $\lambda' = 1/t$.
 
 In our problem, we denote $\eta := 1 - \tau$ and directly choose $\lambda = \mathcal{O}(T^{\eta})$ with $\eta \in (0,1]$, independent of the number of state-action pairs $t-1$.
 We need to remark that, if we choose a more reasonable $\lambda = \mathcal{O}(t^{\eta})$ with $\eta \in (0,1]$, depending on the number of samples, we can still obtain the same regret as \cref{thm:lsvi}.
 To be specific, the regret bound in Eq.~\eqref{eq:regretlambda} is reformulated as (w.h.p)
 \begin{equation*}
    \begin{split}
        {\tt Regret}(T)
        & \lesssim \frac{d+ \lambda}{\gamma \lambda} d^2H^2  \log T  + \frac{d+ \lambda}{\gamma} H^2 \sum_{t=1}^T \sqrt{\frac{\sigma^2}{T}} \\
        & \lesssim \frac{d^2H^2}{\gamma} \log T + \frac{dH^2}{\gamma} T^{-\frac{1}{2}} \int_{1}^T t^{-\eta(\alpha - 1)} \mathrm{d} t \\
         & = \left\{ \begin{array}{rcl}
				\begin{split}
					& \!\!  \mathcal{O} \left(\frac{d^2 H^2}{\gamma} \log T \right) ,~\mbox{if $\eta (\alpha - 1) \geqslant 1/2$} \\
					& \!\! \mathcal{O} \left( \frac{d H^2}{\gamma}  T^{\frac{1}{2}- \eta (\alpha - 1)} \right) ,~\mbox{if $\eta (\alpha \!-\! 1) \in (0,\frac{1}{2})$}  \,.
				\end{split}
			\end{array} \right. 
    \end{split}
\end{equation*}
That means, there is no difference between these two regularization schemes whether it varies with the number of state-action pairs.

Besides, it appears that if we take $\eta=0$, the regularization parameter $\lambda$ is in a constant order, i.e., $\lambda' = \mathcal{O}(1/t)$, decaying fast, we cannot improve the regret rate beyond $\widetilde{\mathcal{O}}(1/\sqrt{T})$. 
It does not make sense in practice.
Here we illustrate this to resolve this issue.

The main reason is, our low variance assumption~\ref{assum:lowv} is based on $\lambda$.
In our theorem, we require $\lambda = T^{\eta}$ with $\eta \in (0,1]$, which makes the feature mapping $\| \phi_h (s_h,a_h) \|^2_{(\Lambda_h^t)^{-1}}$ concentrate around its mean and decay with the episode $T$.
If we take $\eta=0$, the constant order of $\lambda$ does not make $\| \phi_h (s_h,a_h) \|^2_{(\Lambda_h^t)^{-1}}$ decaying with the episode $T$, and accordingly Assumption~\ref{assum:lowv} does not work.
In this case, there is no need to use $\lambda$ as a bridge in our assumption.
Instead, we can directly set $M-m$ small, decaying with $T$ under some certain distribution.

\section{Auxiliary lemma}
In this section, we list some auxiliary lemmas that are needed for our proof.
%The first lemma is related to the GOLF algorithm.

\begin{lemma} \citep[Lemmas 39 and 40]{jin2021bellman}
\label{lem:jin2021}
    Under Assumptions~\ref{assum:rea} and \ref{assum:bellman}, for any $\delta \in (0,1)$, if we choose $\beta = c \log \left( \frac{\mathscr{N}_{\mathcal{F}}(1/T) TH }{ \delta } \right)$ in the GOLF algorithm \ref{alg:golf} for some large constant $c$, with probability at least $1 - \delta$, we have
    \begin{itemize}
        \item $Q^{\star} \in \mathcal{F}^{(t)}$.
        \item $\sum_{i<t} \mathbb{E}_{(s,a) \sim \rho_h^{(i)}} [f_h(s,a) - \mathcal{T}_h f_{h+1}(s,a)]^2 \lesssim \mathcal{O}(\beta)$ for any $f \in \mathcal{F}^{(t)}$.
    \end{itemize}
\end{lemma}

\if 0
To begin with, recall the Freedman's inequality that controls the sum of martingale difference by the sum of their predicted variance.
\begin{lemma}[Freedman's inequality {\citep[e.g.,][]{agarwal2014taming}}]\label{lem:freedman1}
Let $(Z_t)_{t \leq T}$ be a real-valued martingale difference sequence adapted to filtration $\mathfrak{F}_t$, and let $\mathbb{E}_t[\cdot]=\mathbb{E}[\cdot\  | \ \mathfrak{F}_t]$. If $|Z_t|\leq R$ almost surely, then for any $\eta \in (0,\frac{1}{R})$ it holds that with probability at least $1-\delta$,
$$
	\sum_{t=1}^{T}Z_t \leq \mathcal{O} \left( {\eta \sum_{t=1}^{T} \mathbb{E}_{t-1}[Z_t^2]+\frac{\log(\delta^{-1})}{\eta}} \right)\,.
$$
\end{lemma}
\fi 

\begin{lemma}\citep[Bellman error bound for FQI, Lemma 7]{song2022hybrid}
\label{lem:fqioff}
    Let $\delta \in (0,1)$, for any $h \in [H]$ and $t \in [T]$, $f_h^{t+1}$ be the estimated value function computed by the least square regression using samples from $\mathcal{D}_{h}^{\nu} \bigcup \{ (s_h^{\tau}, a_h^{\tau}, s_{h+1}^{\tau})_{\tau = 1}^{t} \}$ in Algorithm~\ref{alg:fqi}, then with probability at least $1 - \delta$, for any $h \in [H-1]$ and $t \in [T]$, we have
    \begin{equation*}
        \mathbb{E}_{\mu_h} \left(\delta_h^{(t)}(s,a)\right)^2 \lesssim \frac{1}{n_{\mathrm{off}}} \log \left( \frac{\mathscr{N}_{\mathcal{F}}(1/T) TH }{ \delta } \right) \,.
    \end{equation*}
\end{lemma}

\begin{lemma}\citep[Per-state-action elliptic potential lemma, modified version]{xie2023role}
\label{lem:per_sa_ep}
Let $\rho^{(1)}, \rho^{(2)}, \dotsc, \rho^{(T)}$ be an arbitrary sequence of distributions over a set $\mathcal{Z}$ (e.g., $\mathcal{Z} = \mathcal{S} \times \mathcal{A} $), and let $\mu\in\Delta(\mathcal{Z})$ be a distribution such that $\rho^{(t)}(z)  \leqslant [C \mu(z)]^p$ for some $p \geqslant 1$ and all $(z,t) \in \mathcal{Z} \times [T]$. Then for all $z \in \mathcal{Z}$, we have
\begin{align*}
\sum_{t = 1}^{T} \frac{d^{(t)}(z)}{\sum_{i < t} d^{(i)}(z) + C \cdot \mu(z)} \leq \mathcal{O} \left(\log\left( T \right) \right)\,.
\end{align*}
\end{lemma}

\end{document}